\documentclass[11pt]{article}
\usepackage{amsmath,amssymb,amsthm,amsfonts}
\usepackage{latexsym,bbm,graphicx,float,mathtools}
\usepackage{algorithm}
\usepackage[noend]{algpseudocode}
\usepackage{framed}
\usepackage{verbatim}
\usepackage[shortlabels]{enumitem}
\usepackage{hyperref}
\usepackage{cleveref}
\usepackage[round]{natbib}
\usepackage{subcaption}
\usepackage[colorinlistoftodos,prependcaption,textsize=small]{todonotes}
\hypersetup{
    colorlinks,
    citecolor=blue,
    filecolor=black,
    linkcolor=black,
    urlcolor=black,
    linktoc=all
}

\newtheorem{theorem}{Theorem}[section]
\newtheorem{lemma}{Lemma}[section]

\newtheorem{proposition}[lemma]{Proposition}

\newtheorem{corollary}[lemma]{Corollary}

\newtheorem{problem}{Problem}

\newtheorem{definition}{Definition}

\newcommand{\costmax}{C_{\max}}

\newcommand{\calA}{\mathcal A}

\newcommand{\calS}{\mathcal S}

\newcommand{\calD}{\mathcal{D}}
\newcommand{\calP}{\mathcal{P}}

\newcommand{\calG}{\mathcal{G}}

\newcommand{\E}{\mathbb{E}}

\newcommand{\TV}{\mathrm{TV}}

\usepackage[dvipsnames]{xcolor}
\newcommand{\Regret}{\mathrm{Regret}}
\newcommand{\SeqRejectron}{\mathsf{SeqRejectron}}

\title{Learning When to Stop: Selective Imitation Learning Under Arbitrary Dynamics Shift}
\author{Surbhi Goel, Jonathan Pei, James Wang}

\begin{document}

\bibliographystyle{plainnat}

\maketitle

\begin{abstract}
Behavior cloning provides strong imitation learning guarantees when training and test environments share the same dynamics. However, in many deployment settings the test environment's transitions differ from training, and classical offline IL offers no recourse: the learner must commit to an action at every state, even when its demonstrations are uninformative and could lead to arbitrary degradation of performance. This motivates the study of \emph{selective} imitation, where the learner may choose to \emph{stop} when it cannot act reliably. We introduce a model for selective imitation under arbitrary dynamics shift: given labeled expert demonstrations from a training environment and unlabeled state trajectories from the same expert in a test environment, the learner outputs a \emph{selective policy} that is \emph{complete} (rarely stops in training) and \emph{sound} (incurs low regret before stopping in test). Our algorithm, $\SeqRejectron$, constructs a stopping rule using a small set of \emph{validator policies} whose size is independent of the horizon or policy class. For deterministic policies, this yields horizon-free $\tilde{O}(\log|\Pi|/\epsilon^2)$ sample complexity, assuming sparse costs. For stochastic policies, we obtain analogous horizon-free guarantees using a cumulative Hellinger stopping time. We extend the framework to misspecified experts and different expert policies across train and test and obtain results that gracefully degrade with the amount of misspecification.
\end{abstract}

\section{Introduction} \label{sec: introduction}
A central challenge in deploying learned policies is \emph{environment shift}: the dynamics governing state transitions at test time may differ from those encountered during training. This arises ubiquitously: in sim-to-real transfer for robotics \citep{tobin2017domain}, in autonomous driving where conditions vary between simulation and the road \citep{codevilla2019exploring}, and in any domain where a policy trained on demonstrations must act in a changed environment.

Classical imitation learning offers strong guarantees when training and test environments coincide, but these guarantees become vacuous under even modest dynamics shift. Even with infinite data, a policy may perform arbitrarily poorly if the test environment's state support differs from the training environment's. The fundamental issue is that it cannot recognize when its training data no longer provides reliable guidance, and has no mechanism to stop before things go wrong. This motivates a basic question: \emph{can a policy learn not only what to do, but when to stop?}

We develop a framework for \emph{learning to imitate with the option to abstain}, drawing on the \emph{PQ learning} paradigm from selective classification \citep{goldwasser2020beyond}. We consider a practical setting where the learner is given labeled (state, action) expert demonstrations from a training environment $M$, but only unlabeled (state-only) expert trajectories from a test environment $N$. This asymmetry naturally models settings where passive observation is cheap and abundant, but action telemetry is costly or proprietary. For instance, in autonomous driving, a learner might train on rich kinematic data from a sensor-equipped fleet, but must generalize to new cities using massive datasets of public dashcam footage---which provide valid visual states of expert driving, but completely lack the underlying steering commands \citep{xu2017end, torabi2018behavioral, codevilla2019exploring}. In healthcare, a sepsis treatment policy trained on a fully instrumented ICU system may be deployed to a setting where precise dosage decisions are unavailable \citep{komorowski2018artificial, finlayson2021clinician}. In algorithmic trading, a model trained on historical market conditions with full access to proprietary executed trades must often navigate novel macroeconomic regimes where only public ticker data is observable, leaving the exact buy and sell actions of top experts completely hidden.

The goal is to learn a \emph{selective policy} that acts when its training data supports confident prediction and stops execution otherwise. We require \emph{completeness} (the policy rarely stops when deployed in the training environment) and \emph{soundness} (whenever the policy chooses to act in the test environment, it matches the expert's performance). A policy that always stops is trivially sound but useless; one that never stops is trivially complete but unsafe. The tension makes the problem nontrivial.

\subsection{Our approach and results}

We certify trajectory prefixes using \emph{validator policies}: competing explanations of the expert drawn from the training-consistent version space. The learner executes a base policy $\pi_0$, continuing only while every validator agrees and abstaining at the first disagreement. Because a small validator set always suffices, the resulting sample complexity depends on the horizon $H$ only through the trajectory cost scale $\costmax$ and the complexity of the policy class. We call this approach $\SeqRejectron$. It generalizes the selective classification algorithm of \citet{goldwasser2020beyond} to sequential decision-making, matching recent horizon-independent rates for imitation learning \citep{foster2024behavior} while accommodating dynamics shift. We analyze it in three increasingly general settings.

\paragraph{Deterministic policies (\Cref{section: deterministic}).} We first develop our framework for finite classes of deterministic policies. In this setting, each validator triggers at the first state where it disagrees with the base policy. We prove that a surprisingly sparse set of these validators is always sufficient to safely stop execution. Consequently, we achieve horizon-free sample complexity: when $m\asymp n$, both the stopping rate and the regret scale as $\tilde O(\sqrt{\log|\Pi|/n})$, where the horizon enters only through the trajectory cost scale $\costmax$. Finally, we demonstrate that this approach is oracle-efficient by providing an exact reduction to weighted realizable multiple-instance learning.

\paragraph{Stochastic policies (\Cref{sec: stochastic}).} The binary disagreement test is too coarse for stochastic policies, where two distributions may differ softly at every step without any single step being alarming. We replace it with a cumulative Hellinger stopping time that tracks a trajectory-level divergence budget, preserving the validator template. The main result gives finite-samples bound on both the stopping rate and the regret. This bound has both a \emph{cost-driven} and \emph{variance-driven} exponent. We prove an $\Omega(1/\epsilon^3)$ labeled-sample lower bound already in the single-step stochastic setting, matching the cost-driven part of the upper bound and leaving the variance-driven exponent open.

\paragraph{Misspecified policy classes (\Cref{sec: misspecification}).} Under misspecification, the exact-consistency version space does not contain the expert. We replace it with a symmetrically regularized game that softly penalizes source-data disagreement, ensuring the validator distribution remains well-defined and sparse. We show that the guarantees degrade gracefully under misspecification.

\subsection{Related work}
Our work builds on three strands of literature.\footnote{We defer a more comprehensive literature review to \Cref{sec: additional related work}.} First, it is closely related to imitation learning and recent analyses of behavior cloning. Behavior cloning reduces imitation learning to supervised prediction, but is classically limited by compounding error under distribution shift \citep{pomerleau1988alvinn,ross2010efficient, ross2011reduction}. Recent theory has further characterized the statistical limits of offline imitation learning and identified settings in which horizon-independent guarantees are possible \citep{rajaraman2020toward, swamy2022minimax, foster2024behavior}. We study a modified setting, where the learner is trained offline, but the test dynamics may shift arbitrarily. 

Second, our framework is inspired by selective classification and learning under arbitrary distribution shift. Closest to our setting is the PQ framework, which studies source-labeled, target-unlabeled learning with abstention under arbitrary shift \citep{goldwasser2020beyond,kalai2021efficient,goel2024tolerant}.
Our framework can be viewed as a sequential analogue of PQ learning, where the learner abstains along a trajectory rather than on a single prediction.

Third, our work is related to robustness and uncertainty in reinforcement learning and robotics, including robust control, sim-to-real transfer, domain randomization, and OOD detection \citep{iyengar2005robust,nilim2005robust,wiesemann2013robust,tobin2017domain,peng2018sim,chebotar2019closing,chae2022robust,haider2023out}. These approaches aim to make policies robust or adaptive to changing environments, but are typically heuristic, make strong assumptions about the nature of the shift, or require online access to the test distribution. By contrast, we study distribution-free guarantees under arbitrary dynamics shift using only labeled source demonstrations and unlabeled target trajectories, with abstention\footnote{Throughout, will will interchangeably refer to \emph{abstention} as \emph{stopping}.} as the safety mechanism.

\section{Problem Formulation} \label{sec: problem setup}
\paragraph{Markov decision processes.}
A finite-horizon Markov decision process (MDP) is a tuple $M = (\calS, \calA, P, c, H)$, where $\calS$ is the state space, $\calA$ is the action space\footnote{To simplify presentation, we assume throughout that $\calS$ and $\calA$ are countable; the definitions and results extend to general spaces with the appropriate measure-theoretic treatment.}, $H$ is the horizon, $P = \{P_h\}_{h=0}^{H-1}$ specifies the state evolution, and $c = \{c_h\}_{h=1}^H$ are per-step cost functions $c_h : \calS \times \calA \to [0,1]$. Here $P_0 \in \Delta(\calS)$ is the initial-state distribution, and for each $h = 1,\dots,H-1$, $P_h(\cdot \mid s,a) \in \Delta(\calS)$ is the transition kernel. A (possibly nonstationary, possibly randomized) policy is $\pi = \{\pi_h\}_{h=1}^H$ with $\pi_h : \calS \to \Delta(\calA)$. Together, $(M,\pi)$ induce a distribution over trajectories $(s_1,a_1),\dots,(s_H,a_H)$ via
\[
s_1 \sim P_0, \qquad a_h \sim \pi_h(\cdot \mid s_h), \qquad s_{h+1} \sim P_h(\cdot \mid s_h,a_h) \quad \text{for } h=1,\dots,H-1.
\]
We denote the expected cumulative cost of $\pi$ in $M$ by $J_M(\pi;c)\coloneqq\E^{M,\pi}\left[\sum_{h=1}^H c_h(s_h,a_h)\right]$.

\paragraph{Stopping times.}
A key feature of our framework is that the learner may stop execution mid-trajectory. We formalize this decision to abstain via a \emph{stopping time} adapted to the observed history. Denote the canonical sample space of full state-action trajectories by $\Omega = (\calS \times \calA)^H$. Let $\mathbb{G} = \{\mathcal{G}_h\}_{h=1}^H$ be the pre-action filtration on $\Omega$, where $\mathcal{G}_h=\sigma(s_1,a_1,\dots,s_{h-1},a_{h-1},s_h)$. Thus $\mathcal{G}_h$ contains exactly the information available at step $h$, just before the learner chooses $a_h$. For convenience, we also let $\mathcal{G}_{H+1} = \sigma(s_1,a_1,\dots,s_H,a_H)$.

We define the point of abstention as a stopping time $\tau : \Omega \to \{1,\dots,H,H+1\}$ adapted to the filtration $\mathbb{G}$. The event $\{\tau = h\}$ signifies that the learner stops execution at step $h$, having observed the history up to $s_h$. We adopt the convention that $\tau = H+1$ if the learner completes the full trajectory without abstaining.

\subsection{Problem statement}

Let $M \coloneqq (\calS,\calA,P,c,H)$ and $N \coloneqq (\calS,\calA,Q,c,H)$ be MDPs that share state space, action space, costs, and horizon, but may differ in their dynamics $P = \{P_h\}_{h=0}^{H-1}$ and $Q = \{Q_h\}_{h=0}^{H-1}$. We refer to $M$ as the \emph{training environment} and $N$ as the \emph{test environment}. For any environment $E \in \{M,N\}$ and policy $\pi$, we write $\calP_E^\pi$ for the law of the trajectory under $(E,\pi)$, and $\calP_{E,\mathrm{state}}^\pi$ for the law of the state trajectory $s_{1:H}$. Let $\costmax$ denote an upper bound on the total trajectory cost $\sum_{h=1}^H c_h(s_h,a_h)$.

Let $\Pi$ be a policy class and let $\pi^{\star}\in \Pi$ denote an expert policy whose behavior we seek to imitate. We do not assume that $\pi^\star$ is optimal for $c$. The learner observes demonstrations generated by $\pi^\star$, but does not observe realized costs or expert actions in $N$:
\begin{itemize}
    \item \textbf{Labeled training rollouts:} drawn i.i.d.\ from $(M,\pi^\star)$, $\mathcal{D}_{\mathrm{train}}=\{(s_{k,1:H},a_{k,1:H})\}_{k=1}^m$.
    \item \textbf{Unlabeled test rollouts:} state-only trajectories drawn i.i.d.\ from $(N,\pi^\star)$, $\mathcal{D}_{\mathrm{test}}=\{t_{l,1:H}\}_{l=1}^n$.
\end{itemize}

Without the option to abstain, the problem is poorly posed under environment shift: if the expert-induced state supports of $M$ and $N$ are disjoint, even infinite labeled data from $M$ cannot determine the expert's actions on test-time states. To maintain low regret across all valid expert hypotheses, the learner is forced to \emph{abstain} on unseen test states. We therefore allow the learner to stop execution mid-trajectory. We formalize this via the notion of a \emph{selective policy}.

\begin{definition}[Selective policy, stopped regret, and stopping rate]
A \emph{selective policy} is a pair $(\pi, \tau)$, where $\pi \in \Pi$ is a policy and $\tau$ is a stopping time adapted to the pre-action filtration $\mathcal G_h = \sigma(s_1, a_1, \dots, s_{h - 1}, a_{h - 1}, s_h)$. When deployed, it incurs cost only on the time steps before $\tau$. We define the \emph{stopped cost} and \emph{stopped regret} (in the test environment $N$) as
\begin{align}
    J_N(\pi,\tau;c) &\coloneqq \E^{N,\pi}\left[\sum_{h=1}^{\tau-1} c_h(s_h,a_h)\right], \qquad \Regret_N(\pi,\tau;c) \coloneqq J_N(\pi,\tau;c) - J_N(\pi^\star, \tau; c).
\end{align}
Additionally, we define the \emph{stopping rate} (in the training environment $M$) as 
\begin{align}
    \alpha_M(\pi,\tau) \coloneqq \Pr_{M,\pi}(\tau \le H).
\end{align}
\end{definition}

Intuitively, a selective policy can be thought of as an agent equipped with a fail-safe, executing its learned behavior only as long as it remains confident. For example, consider a self-driving car trained exclusively to navigate dry, sunny roads. If it is suddenly deployed in a blizzard where tire traction and visibility drastically change, a standard imitation policy might blindly attempt a normal turn and crash. A selective policy, by contrast, recognizes that these icy conditions fall outside its training support and safely pulls over before a catastrophic error occurs.
\begin{problem}[Selective imitation learning: stopped regret formulation]\label{prob: selective imitation}
Given labeled training rollouts $\mathcal{D}_{\mathrm{train}}$ and unlabeled test rollouts $\mathcal{D}_{\mathrm{test}}$, output a selective policy $(\pi,\tau)$ satisfying:
\begin{itemize}
    \item \textbf{Completeness:} $\alpha_M(\pi,\tau) \le \epsilon$. (The policy rarely abstains in $M$).
    \item \textbf{Soundness:} $\Regret_N(\pi,\tau;c) \le \epsilon$. (The policy incurs low stopped regret in $N$).
\end{itemize}
\end{problem}

\section{Selective Execution for Deterministic Policies}\label{section: deterministic}
We first present our results for the case when $\pi^{\star}$ and $\Pi$ are deterministic policies. We first analyze a per-step baseline that suffers an avoidable horizon factor $H$ by treating rollouts as independent decisions at each time step. We then resolve this via a trajectory-level construction based on \emph{validator policies} that certifies entire execution prefixes, yielding an improved sample complexity.

\subsection{First attempt: Stepwise reduction to PQ-learning}\label{subsec: deterministic policies first attempt}

Given that our framework is a natural generalization of PQ-learning \citep{goldwasser2020beyond} to sequential decision making, a natural question is to ask if it is possible to formulate the problem at \emph{each step} as an independent PQ-learning problem. A naive baseline applies Rejectron \cite{goldwasser2020beyond} independently at each step $h \in [H]$, yielding a base policy $\tilde{\pi}_h$ and selection set $S_h$. The learner executes $\tilde{\pi}$ and abstains at the stopping time $\tau = \min\{h : s_h \notin S_h\}$.

Running Rejectron with tolerance $\gamma$ gives $O(\gamma)$ train abstention, $O(\gamma)$ train error, and $O(\gamma)$ test selective error \emph{per step}. However, to convert these to a \emph{trajectory-wise} guarantee, we must union bound over $H$ steps. To achieve regret $\epsilon$, we need $\gamma = O(\epsilon /
(H\costmax))$, which leads to a requirement of $\tilde{O}\left(\frac{H^2 \costmax^2 \log|\Pi|}{\epsilon^2}\right)$ trajectories. For completeness, we formalize this argument in \Cref{subsec: per step pq}.

This $H^2$ penalty persists even when costs are sparse ($\costmax = O(1)$), whereas offline behavior cloning has horizon-independent trajectory complexity under sparse costs~\citep{foster2024behavior}.

\subsection{Trajectory-level validation via validator policies}\label{subsec: deterministic seqrejectron}

The per-step reduction treats a trajectory as $H$ independent decisions, introducing an avoidable horizon factor. However, sequential rollouts have inherent structure: until the learner diverges from the expert, stopping incurs only a prefix of the expert's cost. The danger is failing to stop before this first unrejected deviation. Therefore, our goal is to certify an entire execution prefix at once.

\subsubsection{Validator sets and the existence of sparse validator distributions}

Fix a base policy $\pi_0 \in \Pi_{\mathsf{version}}$. To determine when it is safe to continue executing $\pi_0$, we compare it against other policies in the version space. For any $\pi \in \Pi_{\mathsf{version}}$, the first state where $\pi$ disagrees with $\pi_0$ marks a point where $\pi_0$ is no longer justified by that competing explanation of the expert. Since no single competitor reliably detects all problematic rollouts across all trajectories, we compare $\pi_0$ against a \emph{set} of validators, continuing only while every validator agrees with $\pi_0$.

\begin{definition}[Validator-induced stopping time]
For any base policy $\pi_0$ and set of policies $\Phi \subseteq \Pi_{\mathsf{version}}$, define the stopping time $\tau_{\pi_0, \Phi}$ as the first time along the realized trajectory at which at least one validator in $\Phi$ disagrees with $\pi_0$, or $H+1$ if no such disagreement occurs:\begin{align}
    \tau_{\pi_0, \Phi}(T) \coloneqq \min\left(\big\{h \in [H] : \exists \pi \in \Phi,\ \pi_h(s_h) \neq \pi_{0,h}(s_h)\big\} \cup\{H+1\}\right).\label{eq: validator stopping time}
\end{align}
\end{definition}

The key question is therefore how to choose the validator set $\Phi$. Ideally, we would like $\Phi$ to be such that $\Pr_{\pi_0, N}[\tau_{\pi_0, \Phi} \leq \tau_{\pi_0,\{\pi^{\star}\}}]$ is close to $1$. That is, with high probability, $\Phi$ tells us to stop before $\pi_0$ deviates from $\pi^{\star}$. This tells us that until the stopping time $\tau_{\pi_0, \Phi}$, the policies $\pi_0$ and $\pi^{\star}$ are approximately coupled, and hence incur similar cost. We would also like $\Phi$ to be sparse; this prevents us from rejecting too many trajectories, which is a form of overfitting. 

The key difficulty, however, is that $\Phi$ must be computed \emph{without knowledge of $\pi^{\star}$}. We will thus ask: \emph{Can we find a sparse validator set which stops before \emph{any} policy in $\Pi_{\mathsf{version}}$ deviates from $\pi_0$?} Unfortunately, there is no obvious reason that one small deterministic validator set should work uniformly against every policy in $\Pi_\mathsf{version}$. Fortunately, the following result shows that it is always possible to compute a \emph{distribution} over sparse validator sets which satisfies this guarantee with high probability. We thus sample from this distribution to obtain a sparse validator set for $\pi^{\star}$.

\begin{definition}[$\rho$-valid validator distribution]
Fix $\pi_0 \in \Pi_{\mathsf{version}}$ and unlabeled test trajectories $\mathcal{D}_{\mathsf{test}}=\{T_j\}_{j=1}^n$. A distribution $q$ over subsets $\Phi \subseteq \Pi_{\mathsf{version}}$ is called \emph{$\rho$-valid} for $\pi_0$ if for every $\pi \in \Pi_{\mathsf{version}}$,
\[
\E_{\Phi \sim q}\left[\frac{1}{n}\sum_{j=1}^n \mathbf{1}\left[\tau_{\pi_0,\Phi}(T_j) > \tau_{\pi_0,\{\pi\}}(T_j)\right]\right] \le \rho.
\]
\end{definition}

In words, a $\rho$-valid distribution places mass on validator sets that, on almost all empirical test trajectories, stop no later than when \emph{any} fixed competing policy in the version space diverges from $\pi_0$. The next lemma shows that this averaged notion of coverage can always be achieved by distributions supported on \emph{small} validator sets. Importantly, the support-size bound depends only on the target violation level $\rho$; downstream, $\rho$ will only depend on $\epsilon$ and $\costmax$.

\begin{lemma}[Existence of sparse valid validator distributions]
\label{lemma:sparse-valid-validator-distribution}
For every $\rho \in (0,1)$, there exists a $\rho$-valid distribution $q^\star$ over subsets $\Phi \subseteq \Pi_{\mathsf{version}}$ such that every $\Phi \in \mathsf{supp}(q^\star)$ satisfies $|\Phi| \le \left\lceil \frac{1}{\rho} \right\rceil$.
\end{lemma}\begin{proof}[Proof sketch]
We set up a zero-sum game between a challenger, who picks a distribution $p$ over competitor policies $\pi \in \Pi_{\mathsf{version}}$, and a validator, who picks a distribution $q$ over sets $\Phi$ of $K = \lceil 1/\rho \rceil$ policies from $\Pi_{\mathsf{version}}$. The payoff is the expected fraction of test trajectories on which the validator set stops strictly later than the challenger's policy. Since the payoff is bilinear in $(p, q)$, by the minimax theorem, it suffices to prove the following dual statement: for every $p \in \Delta(\Pi_{\mathsf{version}})$, there exists a validator strategy achieving
\begin{align*}
    \E_{\pi \sim p}\, \E_{\Phi \sim q}\left[\frac{1}{n}\sum_{j=1}^n \mathbf{1}\left[\tau_{\pi_0,\Phi}(T_j) > \tau_{\pi_0,\{\pi\}}(T_j)\right]\right] \leq \frac{1}{K+1} < \rho.
\end{align*}
To prove this, fix any $p$ and consider the validator strategy $p^K$ of drawing all $K$ validators i.i.d.\ from $p$. The challenger and the $K$ validators are then $K+1$ exchangeable i.i.d.\ draws from $p$, so by symmetry the probability that the challenger's stopping time is the strict minimum on any trajectory is at most $1/(K+1)$. Sparsity holds because every $\Phi \in \mathsf{supp}(q^\star)$ has size $K$.
\end{proof}

The proof uses two key ingredients: the minimax theorem and the exchangeability of i.i.d. draws. Importantly, there is no dependence on the state space, action space, horizon, or size of $\Pi$. This generality is important, and we will reuse it for stochastic policies (\Cref{sec: stochastic}) and the misspecified setting (\Cref{sec: misspecification}). We state the general version as \Cref{lemma: generalized sparse selective policies} in the appendix.

\paragraph{How do we compute $q^{\star}$?} To keep the statistical argument separate from the computation, we package \Cref{lemma:sparse-valid-validator-distribution} into a subroutine $\textsc{SparseValidatorDist}(\Pi_{\mathsf{version}},\pi_0,\mathcal{D}_{\mathsf{test}},\rho, \xi,\delta)$, which with probability $1-\delta$, returns a $(\rho + \xi)$-valid distribution over subsets of $\Pi_{\mathsf{version}}$ with size at most $\left\lceil \frac1\rho \right\rceil$. Here $\xi$ is a computational slack term that can be driven arbitrarily close to 0. A concrete implementation is given in \Cref{section: implementation}, where we realize this subroutine via no-regret dynamics.

\subsubsection{Algorithm and guarantees}

We are now ready to present our algorithm for selective imitation with deterministic policies. The algorithm (\Cref{alg: mdprejectron}) has a clean, simple structure. The labeled data define a version space and a base policy $\pi_0$. The unlabeled test trajectories are used to construct a sparse validator distribution, from which we sample a validator set. The returned selector continues while every validator agrees with $\pi_0$ and abstains at the first disagreement.

\begin{algorithm}[H]
\caption{$\SeqRejectron$ for deterministic policies}
\label{alg: mdprejectron}
\begin{algorithmic}[1]
\Require deterministic policy class $\Pi$, labeled rollouts $\mathcal{D}_{\mathsf{train}}=\{(s_{ih},a_{ih})\}_{i\in[m],\,h\in[H]}$, unlabeled rollouts $\mathcal{D}_{\mathsf{test}}=\{T_j\}_{j=1}^n$, tolerance $\eta$, computational slack $\xi$, confidence $\delta$
\State Compute the training-consistent version space
\[
\Pi_{\mathsf{version}} \gets \left\{\pi \in \Pi : \pi_h(s_{ih})=a_{ih}\ \text{for all } i\in[m],\ h\in[H]\right\}.
\]
\State Choose any base policy $\pi_0 \in \Pi_{\mathsf{version}}$.
\State Let $q^{\star}\gets \textsc{SparseValidatorDist}(\Pi_{\mathsf{version}},\pi_0,\mathcal{D}_{\mathsf{test}},\eta/2, \xi, \delta/5)$.
\State Sample $\Phi_1,\dots,\Phi_k \stackrel{\mathrm{i.i.d.}}{\sim} q^\star$ with $k \gets \lceil \log_2(5/\delta)\rceil$ and set $\Phi \gets \bigcup_{r=1}^k \Phi_r$.
\State \Return the selective policy $(\pi_0,\tau_{\pi_0, \Phi})$.
\end{algorithmic}
\end{algorithm}

Our main result for this section, \Cref{theorem: deterministic rejection sample complexity}, is a sample complexity for \Cref{alg: mdprejectron}. 

\begin{theorem}[$\SeqRejectron$ guarantee for deterministic classes]
\label{theorem: deterministic rejection sample complexity}
Let $\eta,\delta, \xi>0$. Define $Z \coloneqq (\lceil\log_2(5/\delta)\rceil\lceil 2/\eta \rceil+1)\log|\Pi| + \log(5/\delta)$. Then, with probability at least $1-\delta$, the selective policy $(\pi_0,\tau_{\pi_0, \Phi})$ returned by \Cref{alg: mdprejectron} satisfies
\[
\alpha_M(\pi_0,\tau_{\pi_0, \Phi}) \le \frac{2Z}{m}
\qquad\text{and}\qquad
\Regret_N(\pi_0,\tau_{\pi_0, \Phi};c) \le \costmax\left(\eta + 2\xi + \sqrt{\frac{2(\eta+2\xi) Z}{n}} + \frac{3Z}{n}\right).
\]
\end{theorem}

\begin{corollary}\label{cor: deterministic optimized rates}
Suppose $m = n \ge 6\lceil\log_2(5/\delta)\rceil\log|\Pi|$. Set $\eta \coloneqq \sqrt{6\lceil\log_2(5/\delta)\rceil\log|\Pi|/n}$. Then, with probability at least $1-\delta$, both the stopping rate and $\Regret_N/\costmax$ scale as $\tilde O(\sqrt{\log|\Pi|/n})$.
\end{corollary}

\paragraph{The role of weight sharing in $\log |\Pi|$.} The sample complexity of \Cref{theorem: deterministic rejection sample complexity} scales with $\log|\Pi|$. When $\Pi$ shares parameters across time steps (e.g., a single neural network mapping states to actions), $\log|\Pi|$ is independent of $H$ and the bound is horizon-free in the strongest sense. Without weight sharing---e.g., $\Pi = \Pi_1 \times \cdots \times \Pi_H$---one has $\log|\Pi| = \sum_{h=1}^H \log|\Pi_h|$, and a linear horizon factor re-enters through the complexity term. In general, $H$ enters only through $\costmax$ and $\log|\Pi|$, with no additional dependence, consistent with \cite{foster2024behavior}.

\paragraph{Bounding the test-side abstention.} It is also natural to ask how often the selective policy abstains in $N$. When the environment shift is bounded in total variation, the test-side stopping rate $\alpha_N(\pi_0,\tau)$ inherits the training-side guarantees up to a TV correction term, plus an \emph{additional} sequential penalty that arises because the deployed policy's actions can steer the trajectory out of the safe prefix before the validator stops. See \Cref{subsec: test-side abstention} in the appendix for a precise statement.

\subsection{Constructing a validator distribution using no-regret dynamics}\label{section: implementation}

The nonconstructive minimax proof of \Cref{lemma:sparse-valid-validator-distribution} can be made algorithmic via the well-known connection between approximate minimax equilibria and no-regret play. We formulate the problem as a repeated game where a maximization player maintains a distribution over the version space using a no-regret algorithm, and the validation player plays an (approximate) best-response by sampling a sparse validator set. Averaging the validation player's choices over $T$ rounds yields a distribution that converges to the minimax optimum. This gives an explicit, oracle-efficient implementation of the $\textsc{SparseValidatorDist}$ subroutine.

Fix the version space $\Pi_{\mathsf{version}}$, a base policy $\pi_0 \in \Pi_{\mathsf{version}}$, and unlabeled test trajectories $\mathcal{D}_{\mathsf{test}}=\{T_j\}_{j=1}^n$. Write each test trajectory as $T_j = (s_{j1},\ldots,s_{jH})$.

\begin{algorithm}[H]
\caption{$\textsc{SparseValidatorDist}(\Pi_{\mathsf{version}},\pi_0,\mathcal{D}_{\mathsf{test}},\rho,\xi, \delta)$}
\label{alg:sparse-validator-dist}
\begin{algorithmic}[1]
\Require version space $\Pi_{\mathsf{version}}$, base policy $\pi_0$, unlabeled test trajectories $\mathcal{D}_{\mathsf{test}}=\{T_j\}_{j=1}^n$, tolerance $\rho \in (0,1)$, computational slack $\xi$, confidence $\delta$, no-regret algorithm $\mathcal A$ with regret bound $\mathrm{Reg}_T$.
\State Set $K \gets \lceil 1/\rho \rceil$.
\State Set $T$ large enough that $\frac{\mathrm{Reg}_T}{T}+\sqrt{\frac{\log(1/\delta)}{2T}} \le \xi $.
\State Initialize any distribution $p^1$ over $\Pi_{\mathsf{version}}$.
\For{$t=1,\dots,T$}
    \State Draw $\pi_1^t,\dots,\pi_K^t \overset{\mathrm{i.i.d.}}{\sim} p^t$ and set $\Phi^t \gets \{\pi_1^t,\dots,\pi_K^t\}$.
    \State For each $\pi \in \Pi_{\mathsf{version}}$, define the payoff $u^t(\pi) \coloneqq \frac{1}{n}\sum_{j=1}^n \mathbf{1}\left[\tau_{\pi_0,\Phi^t}(T_j) > \tau_{\pi_0,\{\pi\}}(T_j)\right]$.
    \State Update $p^{t+1}$ using $\mathcal A$ for maximizing the payoffs $u^t$.
\EndFor
\State Return the empirical distribution $\bar q \coloneqq \frac{1}{T}\sum_{t=1}^T \delta_{\Phi^t}$.
\end{algorithmic}
\end{algorithm}

The interpretation of the payoff is straightforward: $u^t(\pi)$ is the fraction of empirical test trajectories on which the validator set $\Phi^t$ stops strictly later than the single competing policy $\pi$. Thus the maximization player tries to expose competitors against which the current validator set is weak, while the validation player responds by sampling a fresh set from the current maximizer distribution.

\begin{proposition}[Generic no-regret construction of $q^{\star}$]
\label{prop:validator-noregret}
Let $\bar q$ be the output of \Cref{alg:sparse-validator-dist}. With probability at least $1-\delta$, we have
\[
    \sup_{\pi \in \Pi_{\mathsf{version}}}
    \E_{\Phi \sim \bar q}\left[
    \frac{1}{n}\sum_{j=1}^n \mathbf{1}\left[\tau_{\pi_0,\Phi}(T_j) > \tau_{\pi_0, \{\pi \} }(T_j)\right]
    \right]
    \le
    \rho + \xi.
\]
In particular, using Hedge yields $\mathrm{Reg}_T \le \sqrt{2T \log |\Pi|}$, meaning $T = O\left(\frac{\log |\Pi| + \log(1/\delta)}{\rho^{2}}\right)$ rounds suffice to return an $O(\rho)$-valid distribution supported on sets of size at most $\lceil 1/\rho\rceil$.
\end{proposition}

The no-regret construction above is explicit, but inefficient (when using Hedge) when the version space is large. A standard choice in online learning with large action spaces is to use Follow-the-Perturbed-Leader (FTPL) instead of Hedge. However, a naive FTPL implementation would perturb one coordinate per policy in $\Pi_{\mathsf{version}}$, which is also intractable. Instead, we show that by perturbing the \emph{dataset}, and then solving a multiple-instance learning problem \citep{dietterich1997solving, maron1997framework} over this perturbed dataset, we can implement an \emph{oracle-efficient} no-regret algorithm for the maximizer. This places our implementation in the setting of generalized FTPL \citep{dudik2020oracle}. The full formulation and proofs of this reduction are deferred to \Cref{subsec: proof of oracle efficiency}.

\section{Selective Execution for Stochastic Policies}\label{sec: stochastic}
Throughout this section, $\Pi$ denotes a finite class of (possibly stochastic, possibly nonstationary) policies, and $\pi^{\star}\in \Pi$ is the expert. Recall that the deterministic construction rests on a binary primitive: two policies either agree or disagree at each state. For stochastic policies this is too coarse, since two distributions may differ softly at every step without any single discrepancy being alarming. The right question is whether soft discrepancies \emph{accumulate} enough to cause meaningful divergence from the expert. We therefore replace three ingredients from the deterministic construction. 

First, we take $\pi_0$ to be the MLE and for $\gamma > 0$ define the version space as a log-loss ball,\begin{align*}
    &\Pi_{\mathsf{version}}^\gamma \coloneqq \bigl\{\pi \in \Pi : \mathsf{LogLoss}(\pi,\calD_{\mathsf{train}}) \le \mathsf{LogLoss}(\pi_0,\calD_{\mathsf{train}}) + \gamma\bigr\},\\
    &\text{where}\quad \mathsf{LogLoss}(\pi,\calD_{\mathsf{train}}) \coloneqq -\frac{1}{m}\sum_{i=1}^m \sum_{h=1}^H \log \pi_h(a_{ih} \mid s_{ih}).
\end{align*}

Second, we replace the first-disagreement stopping time with one based on cumulative squared Hellinger distance\footnote{For deterministic $\Pi$, $d_H^2$ reduces to the disagreement indicator, so $\tau_{\pi_0, \Phi}^\theta$ with $\theta < 1$ recovers \eqref{eq: validator stopping time}. }: for $p,q \in \Delta(\calA)$, let $d_H^2(p,q) := 1 - \sum_a \sqrt{p(a)q(a)}$, and for $\theta > 0$ define\begin{align}
    \tau_{\pi_0, \Phi}^\theta(T) \coloneqq \min\left(\left\{h \in [H] : \exists\pi \in \Phi,\; \sum_{k=1}^h d_H^2\bigl(\pi_k(\cdot\mid s_k),\pi_{0,k}(\cdot\mid s_k)\bigr) > \theta \right\}\cup\{H+1\}\right),\label{eq: stochastic stopping time}
\end{align}

Third, we define an analog to the \textsc{SparseValidatorDist} subroutine: \textsc{StochasticSparseValidatorDist}, which when called with parameters $\rho$, $\theta$, and $\delta$ returns a distribution $q^{\star}$ supported over subsets $\Phi \subseteq \Pi_\mathsf{version}^\gamma$ of size at most $\lceil 1/\rho\rceil$, such that, with probability $1-\delta$, \begin{align*}
    \E_{\Phi \sim q^{\star}}\left[\frac{1}{n}\sum_{j=1}^n \mathbf{1}\left[\tau_{\pi_0,\Phi}^\theta(T_j) > \tau_{\pi_0,\{\pi\}}^\theta(T_j)\right]\right] \le \rho \quad \text{for every $\pi \in \Pi_{\mathsf{version}}^\gamma$.}
\end{align*}
The existence of such a distribution follows the exact same combinatorial argument as \Cref{lemma:sparse-valid-validator-distribution}, and is stated as \Cref{lemma: generalized sparse selective policies} in the appendix.

\subsection{Algorithm and guarantees}

The stochastic $\SeqRejectron$ follows the same template as the deterministic version, save for the changes mentioned above. We present it as \Cref{alg: stochastic mdprejectron} below.

\begin{algorithm}[H]
\caption{$\SeqRejectron$ for stochastic policies}
\label{alg: stochastic mdprejectron}
\begin{algorithmic}[1]
\Require policy class $\Pi$, labeled rollouts $\calD_{\mathsf{train}}=\{(s_{ih},a_{ih})\}_{i\in[m],h\in[H]}$, unlabeled rollouts $\calD_{\mathsf{test}}=\{T_j\}_{j=1}^n$, tolerance $\eta$, confidence $\delta$, threshold $\theta$, version-space radius $\gamma$
\State Compute the MLE base policy $\pi_0 \gets \arg\min_{\pi \in \Pi}\; \mathsf{LogLoss}(\pi,\calD_{\mathsf{train}})$.
\State Form the version space $\Pi_{\mathsf{version}}^\gamma$.
\State Let $q^{\star}\gets \textsc{StochasticSparseValidatorDist}(\Pi_{\mathsf{version}}^\gamma,\pi_0,\calD_{\mathsf{test}},\eta/2, \theta, \delta/4)$.
\State Sample $\Phi_1,\dots,\Phi_k \stackrel{\mathrm{i.i.d.}}{\sim} q^\star$ with $k \gets \lceil \log_2(4/\delta)\rceil$ and set $\Phi \gets \bigcup_{r=1}^k \Phi_r$.
\State \Return the selective policy $(\pi_0, \tau^\theta_{\pi_0, \Phi})$.
\end{algorithmic}
\end{algorithm}

The threshold $\theta$ explicitly trades off the stopping rate against the allowable Hellinger divergence. Crucially, tracking the \emph{cumulative} divergence avoids reintroducing a horizon factor $H$: bounding the stopping rate on $M$ reduces to controlling the expected Hellinger sum over the version space, rather than applying a union bound over individual time steps.

To formalize our guarantees, let $\mathcal P^{\pi_{|\tau}}_N$ denote the distribution of trajectories induced by $\pi$, stopped at time $\tau$ in environment $N$ (i.e., $\mathcal P^{\pi_{|\tau}}_N$ is supported over trajectory \emph{prefixes}).\footnote{More formally, $\mathcal P^{\pi_{|\tau}}_N$ is a probability measure over the space of finite histories $\mathcal{H} = \bigcup_{h=1}^{H+1} (\mathcal{S} \times \mathcal{A})^{h-1} \times \mathcal{S}$. A random draw from this measure takes the form $(s_1, a_1, \dots, s_{\tau-1}, a_{\tau-1}, s_\tau)$ on the event $\{\tau \le H\}$, and is the full trajectory $(s_1, a_1, \dots, s_H, a_H)$ on the event $\{\tau = H+1\}$.} Our main result is a parameter-dependent sample complexity guarantee for \Cref{alg: stochastic mdprejectron}.

\begin{theorem}[$\SeqRejectron$ guarantee for stochastic classes]\label{theorem: stochastic rejection sample complexity}
Let $\eta, \theta,\gamma,\delta > 0$ (see \Cref{alg: stochastic mdprejectron} for a reminder of these parameters). Define $K_{\mathsf{ens}} \coloneqq \lceil\log_2(4/\delta)\rceil\lceil 2/\eta \rceil$ and $Z \coloneqq (K_{\mathsf{ens}}+1)\log|\Pi| + \log(4/\delta)$. Additionally, suppose $n \ge 8Z/\eta$ and $m \ge (\log|\Pi| + \log(8/\delta))/\gamma$. Then, with probability at least $1-\delta$, the selective policy $(\pi_0,\tau^\theta_{\pi_0, \Phi })$ returned by \Cref{alg: stochastic mdprejectron} satisfies $\alpha_M(\pi_0,\tau^\theta_{\pi_0, \Phi }) \le K_{\mathsf{ens}}\left(\frac{12}{\theta}+48\right)\gamma$ and $D_H^2\bigl(\mathcal P^{\pi_{0|\tau}}_N, \mathcal P^{\pi^\star_{|\tau}}_N\bigr) \le 3(\theta + \eta)$.
\end{theorem}

Note that the safety bound is stated as a bound on the Hellinger distance of the \emph{stopped trajectories} induced by $(\pi_0, \tau)$ and $(\pi^{\star}, \tau)$. To convert this to a regret guarantee, a straightforward approach is to convert this to a total variation guarantee via the standard inequality $\TV \leq \sqrt 2 D_H$. By optimizing hyperparameters, this yields the following end-to-end guarantee.

\begin{corollary}\label{corollary: new stopped regret guarantee}
Set $\gamma \coloneqq (\log|\Pi| + \log(8/\delta))/m$. There exists a choice of parameters $\eta = \theta$ (depending only on $m$, $n$, $\sigma_{\pi^\star}^2$, $\log |\Pi|$, and $\delta$) such that with probability at least $1 -\delta$, both $\alpha_M(\pi_0,\tau^\theta_{\pi_0, \Phi })$ and $\Regret_N(\pi_0,\tau^\theta_{\pi_0, \Phi };c) / \costmax$ are $O\left( \max\left\{ \left(\frac{(\log(1/\delta) + \log|\Pi|)^2}{m}\right)^{1/5}, \left(\frac{(\log(1/\delta) + \log|\Pi|)^2}{n}\right)^{1/4} \right\} \right)$.
\end{corollary}

A natural question is whether this bound can be sharpened: \Cref{theorem: stochastic rejection sample complexity} is stated as a guarantee on the squared Hellinger distance of the stopped trajectories, and it seems wasteful to naively convert it to a total variation bound. Indeed, \cite{foster2024behavior} use a variance-dependent Hellinger change-of-measure lemma to avoid this apparent looseness in the \emph{full horizon} setting. Their bounds are stated in terms of an expert variance parameter $\sigma^2_{\pi^{\star}}$; in our setting, however, this variance parameter would depend on the \emph{learned} stopping time. Thus, it is not clear how to make full use of the squared Hellinger guarantee of \Cref{theorem: stochastic rejection sample complexity} in our setting. It is an interesting question for future work whether or not a different algorithm or analysis can sharpen \Cref{corollary: new stopped regret guarantee}.

In the following section, we present a natural variation of our objective (\Cref{prob: selective imitation}) which is more amenable to the sharper Hellinger change-of-measure, and prove sharper rates accordingly.

\paragraph{Computational implementation and offline oracles.} Unlike the deterministic setting, the cumulative Hellinger thresholding required for stochastic policies does not yet admit a known oracle-efficient reduction to standard ERM. We leave the development of computational oracles for this cumulative objective---or rigorous guarantees for practical surrogates---to future work.

\subsection{Expert handoff interpretation and switched regret guarantees}\label{sec: reward extension appendix}

So far, we have studied policies that can stop execution mid-trajectory. However, in collaborative autonomy---such as a self-driving car handing control back to a human driver upon encountering unfamiliar condition---the episode does not end at abstention. We would like to guarantee that the autonomous system does not drive the system into an unrecoverable state before handing off. Thus, one may consider the \emph{switched policy} that follows the learner until $\tau$ and the expert thereafter. 

\begin{definition}[Switched policy and switched regret]
For a base policy $\pi \in \Pi$, an expert policy $\pi^{\star}\in \Pi$, and a stopping time $\tau$, the \emph{switched policy} $\pi^{\mathrm{sw}}[\pi, \tau, \pi^\star]$ executes $\pi$ on steps $1, \dots, \tau - 1$ and hands over control to $\pi^\star$ from step $\tau$ onward. For a given cost sequence $c = \{c_h\}_{h=1}^H$, we define the \emph{switched regret} (in the test environment $N$) as the difference in expected full-horizon cost:
\begin{align}
    \Regret_N^{\mathrm{sw}}(\pi,\tau;c) \coloneqq J_N(\pi^{\mathrm{sw}}[\pi,\tau,\pi^\star];c) - J_N(\pi^\star;c).
\end{align}
\end{definition}

\begin{problem}[Selective imitation learning: switched regret formulation]\label{prob: switched imitation}
Given labeled training rollouts $\mathcal{D}_{\mathrm{train}}$ and unlabeled test rollouts $\mathcal{D}_{\mathrm{test}}$, output a selective policy $(\pi,\tau)$ such that the induced switched policy $\pi^{\mathrm{sw}}[\pi,\tau,\pi^\star]$ satisfies:
\begin{itemize}
    \item \textbf{Completeness:} $\alpha_M(\pi,\tau) \le \epsilon$. (The policy rarely abstains in $M$).
    \item \textbf{Soundness:} $\Regret_N^{\mathrm{sw}}(\pi,\tau;c) \le  \epsilon$. (The policy incurs low switched regret in $N$).
\end{itemize}
\end{problem}

The switched regret is directly controlled by our \Cref{thm: mixed regret bound} which generalizes the Hellinger regret decomposition of \cite{foster2024behavior}. Below, let $\sigma_{\pi^\star}^2 \coloneqq \sum_{h=1}^{H}{\E_{\pi^{\star}}\left[(V_h^{\pi^{\star}}(x_h) - Q_{h}^{\pi^{\star}}(x_h, a_h))^2\right]} \leq \costmax^2$ denote the expert-variance parameter of \citet{foster2024behavior} where $(x_h,a_h)$ is the state-action pair at time $h$ under $\pi^\star$ and $V_h^{\pi^\star},Q_h^{\pi^\star}$ are the corresponding value and action-value functions.

\begin{lemma}[Hellinger-based regret decomposition for stopped trajectories; generalization of Theorem 3.1 of \cite{foster2024behavior}]\label{thm: mixed regret bound}
Assume $\costmax \ge 1$. Let $\pi^\star$ be the expert policy. For any learned policy $\hat{\pi}$ and stopping time $\tau$, let $\pi_{\mathsf{sw}}$ denote the policy that executes $\hat{\pi}$ for $h < \tau$ and executes $\pi^\star$ for all remaining steps $h \geq \tau$. Then, for any $\epsilon \in (0, e^{-1})$, the expected regret of $\pi_\mathsf{sw}$ is bounded by: \begin{align}\label{eq: stopped trajectory regret bound}
    J(\pi_{\mathsf{sw}}) - J(\pi^\star) \le \sqrt{6\sigma_{\pi^\star}^2 \cdot D_H^2\bigl(\mathcal  P^{\hat{\pi}_{|\tau}}_N,\, \mathcal P^{\pi^\star_{|\tau}}_N\bigr)} + O\left(\costmax \log(\costmax\epsilon^{-1})\right) \cdot D_H^2\bigl(\mathcal P^{\hat{\pi}_{|\tau}}_N,\, \mathcal P^{\pi^\star_{|\tau}}_N\bigr) + \epsilon.
\end{align}
\end{lemma}

Optimizing this over $\eta$ and $\theta$ yields end-to-end rates.

\begin{corollary}\label{cor: stochastic optimized rates}
Set $\gamma \coloneqq (\log|\Pi| + \log(8/\delta))/m$. Let $Z\coloneqq \log|\Pi| + \log(8/\delta)$. There exists a choice of parameters $\eta = \theta$ (depending only on $m$, $n$, $\sigma_{\pi^\star}/\costmax$, $|\Pi|$, and $\delta$) such that with probability at least $1 -\delta$: $\alpha_M(\pi_0,\tau^\theta_{\pi_0, \Phi })= O\left( \max\left\{ \left(\frac{(\sigma_{\pi^\star}/\costmax)^4 Z^2}{m}\right)^{1/5}, \left(\frac{Z^2}{m}\right)^{1/3}, \left(\frac{(\sigma_{\pi^\star}/\costmax)^4 Z^2}{n}\right)^{1/4}, \left(\frac{Z^2}{n}\right)^{1/2} \right\} \right) $
and $\Regret_N^{\mathrm{sw}}(\pi_0,\tau^\theta_{\pi_0, \Phi };c)=
 O\left( \max\left\{ \left(\frac{\costmax \sigma_{\pi^\star}^4 Z^2}{m}\right)^{1/5}, \left(\frac{\costmax^3 Z^2}{m}\right)^{1/3}, \left(\frac{\sigma_{\pi^\star}^4 Z^2}{n}\right)^{1/4}, \left(\frac{\costmax^2 Z^2}{n}\right)^{1/2} \right\} \right). $
\end{corollary}

Corollary~\ref{cor: stochastic optimized rates} has several parts. Both the stopping rate and regret bounds are a maximum over four terms. The first two terms are variance-driven and cost-driven contributions from the labeled sample, and the last two are their unlabeled-sample counterparts. Shortly, we will see that the cost-driven labeled-data portion of this upper bound is tight.

\paragraph{Switched regret vs. stopped regret.} Both the switched regret and stopped regret bounds highlight the mechanism behind $\SeqRejectron$: we simultaneously learn a stopping time and align the learner's pre-stopping trajectory distribution with the expert's. This is analogous to the full-horizon measure-matching interpretation of \cite{foster2024behavior} for the log-loss. 

As previously hinted, the switched regret is driven to zero at a faster rate than the stopped regret in the low $\sigma^2_{\pi^{\star}}$ regime. This is perhaps unintuitive, as the switched regret bound provides a qualitatively strong guarantee: that the learner does not drive the system into an unrecoverable state before switching; on the other hand, the stopped regret guarantee only guarantees that the learner is competitive with the expert on the prefix for which it acts. Providing a resolution to this (for example, sharpening Corollary~\ref{corollary: new stopped regret guarantee}) is an interesting direction for future work. \footnote{We suggest a partial resolution: define the \textit{asymmetric stopped regret} as $J_N(\pi, \tau; c) - J_N(\pi^{\star}; c)$---physically, this compares the cost accumulated by the learner before handing over control to the cost the expert would have incurred driving the entire route. It is not hard to see that this is upper bounded by the switched regret, hence our algorithm controls the asymmetric stopped regret at the same rate as the switched regret (i.e., with the same rates as Corollary~\ref{cor: stochastic optimized rates}).}

\subsection{Single-step lower bound for stochastic policies}

How tight are these rates? The upper bound of Corollary~\ref{cor: stochastic optimized rates} has two regimes: a variance-driven $\widetilde O(\epsilon^{-5})$ term and a cost-driven $\widetilde O(\epsilon^{-3})$ term. The next result shows that the cubic dependence is already unavoidable in the single-step setting, matching the cost-driven part.

\begin{theorem}[Stochastic PQ lower bound]\label{theorem: stochastic lower bound}
    There exists a constant $c_0 > 0$ such that: 
    
    For any dimension $d \ge 1$ and tolerance $\epsilon \in (0, 1/16]$, there exist a state space $\mathcal{X}$, a stochastic policy class $\Pi$ of log-cardinality $d$, and fixed known distributions $P$ and $Q$ over $\mathcal{X}$ for which the following holds: For any \emph{proper}\footnote{By proper selective learner, we mean that the base policy $\pi_0$ always lies in $\Pi$. This holds for $\SeqRejectron$. In the batch (i.e., single step) setting this holds, for example, for $\mathsf{Rejectron}$ \citep{goldwasser2020beyond} and the Slice-and-Dice algorithm of \cite{kalai2021efficient}.} selective learner, if the labeled sample size is $m \le c_0 \frac{d}{\epsilon^3}$, then there exists an expert $\pi^{\star}\in \Pi$ and a bounded cost function $c : \mathcal{X} \times \{0,1\} \to [0,1]$ such that
    \[\E\left[\alpha_P(\pi,\tau)\right] > \epsilon \quad \textup{or}\quad \E\left[\Regret_Q(\pi,\tau;c)\right] > \epsilon\]
    where the outer expectation is over the labeled sample and the learner's internal randomization, which in particular determine $\pi$ and $\tau$.
\end{theorem}

\Cref{theorem: stochastic lower bound} shows that the $\epsilon^{-3}$ dependence in labeled sample complexity is already unavoidable in the single-step stochastic setting. Crucially, because this lower bound is established in a single-step environment ($H=1$), there is no future horizon over which errors can accumulate after an abstention. Consequently, the stopped regret and the switched regret are equal in this construction, meaning this lower bound applies to both formulations. It therefore matches the $\costmax$-driven $\widetilde O(\epsilon^{-3})$ portion of Corollary~\ref{cor: stochastic optimized rates}, while leaving a gap to the current variance-driven $\widetilde O(\epsilon^{-5})$ labeled-sample rate. For deterministic classes, the lower bound of \citet{goldwasser2020beyond} for PQ learning implies that the $\tilde{O}(1/\epsilon^2)$ labeled complexity is tight. Whether the sequential stochastic upper bound can be improved, or whether stronger lower bounds are possible, remains open.

\section{Extending to Misspecified Policy Classes} \label{sec: misspecification}
Previous sections assumed the expert $\pi^\star$ lies in the learner's policy class $\Pi$. In practice, experts often employ richer representations, breaking this realizability assumption. Here, we show our validator framework degrades gracefully under misspecification. Throughout this section, let $\Pi$ be a finite class of deterministic policies, and let $\pi^\star$ be a deterministic expert. We do not assume $\pi^{\star}\in \Pi$.

\paragraph{Misspecification benchmark.} Let $\tau$ be as in \eqref{eq: validator stopping time}. For any deterministic policy $\pi \in \Pi$, define\begin{align*}
    d_M(\pi)
    \coloneqq
    \Pr_{M, \pi^{\star}}\big[\tau_{\pi, \{\pi^{\star}\}}\le H\big],
    \qquad
    d_N(\pi)
    \coloneqq
    \Pr_{N, \pi^{\star}}\big[\tau_{\pi, \{\pi^{\star}\}} \le H\big].
\end{align*}
These are the probabilities that $\pi$ deviates from the expert at any time during the trajectory (which has horizon $H$). We define the \emph{policy-specific} and \emph{best-in-class} misspecification, respectively, as\begin{align*}
    \Delta_\pi \coloneqq \max(d_M(\pi), d_N(\pi)), \qquad \Delta \coloneqq \min_{\pi\in\Pi} \Delta_\pi.
\end{align*}
Let $\tilde \pi = \arg\min_{\pi\in \Pi} \Delta_\pi$ be the policy which achieves the best-in-class misspecification.

Finally, we define the empirical analogue of $d_M$ as $\widehat d_M(\pi):=\frac{1}{m}\sum_{i=1}^m \mathbf 1\left[\tau_{\pi, \{\pi^{\star}\}}(S_i)\le H\right]$.

\subsection{Agnostic version of \Cref{alg: mdprejectron} and guarantees}

When the expert lies outside the policy class ($\pi^{\star}\notin \Pi$), the exact-consistency version space used in \Cref{alg: mdprejectron} may be empty, or may simply fail to contain the expert. To ensure the validator set remains well-defined and sparse even under misspecification, we transition from a constrained optimization (forcing zero training disagreement) to a symmetrically regularized game. In this agnostic framework, we allow validators to disagree with the base policy on the source data, but we penalize this disagreement at a rate $\Lambda$. This allows the learner to "trade off" a small amount of source-side abstention for a significant reduction in target-side late-stop risk. This shift is captured by the following equilibrium result, which can be thought of as an ``agnostic" analog of \Cref{lemma:sparse-valid-validator-distribution}:

\begin{lemma}\label{lem:symmetrically-regularized-equilibrium}
Let $\pi_0 \in \arg\min_{\pi \in \Pi} \widehat{d}_M(\pi)$ be an empirical disagreement minimizer. For every penalty parameter $\Lambda > 0$ and integer $K \ge 1$, there exists a distribution $q^\star$ over validator sets $\Phi = (\phi_1, \dots, \phi_K) \in \Pi^K$ that simultaneously satisfies the following two properties:
\begin{enumerate}
    \item The expected training disagreement is bounded by the empirical minimum:\begin{align}
        \mathbb{E}_{\Phi \sim q^\star}\left[ \sum_{k=1}^K \widehat{d}_M(\phi_k) \right] \le K\cdot \widehat{d}_M(\pi_0) + \frac{1}{\Lambda}\label{eq: misspecified completeness}
    \end{align}
    \item For every target policy $\pi \in \Pi$, the expected late-stop risk scales exclusively with its excess empirical risk over the base policy:\begin{align}
        \mathbb{E}_{\Phi \sim q^\star}\left[ \frac{1}{n}\sum_{j=1}^n \mathbf 1\left[\tau_{\pi_0,\Phi}(T_j)>\tau_{\pi_0,\{\pi\}}(T_j)\right]\right] \le \Lambda \Big( \widehat{d}_M(\pi) - \widehat{d}_M(\pi_0) \Big) + \frac{1}{K}\label{eq: misspecified soundness}
    \end{align}
\end{enumerate}
\end{lemma}

Similar to \Cref{subsec: deterministic seqrejectron}, \Cref{lem:symmetrically-regularized-equilibrium} is a non-constructive result, although it may be realized by no-regret play. We package the output \Cref{lem:symmetrically-regularized-equilibrium} into an abstract subroutine, \textsc{RegularizedSparseValidatorDist} and leave the question of its efficient computational implementation to future work. We are now ready to present our algorithm for misspecified deterministic classes.

\begin{algorithm}[H]
\caption{$\SeqRejectron$ for misspecified deterministic policies}
\label{alg: agnostic seqrejectron}
\begin{algorithmic}[1]
\Require deterministic policy class $\Pi$, labeled rollouts $\mathcal{D}_{\mathsf{train}}=\{(s_{ih},a_{ih})\}_{i\in[m],\,h\in[H]}$, unlabeled rollouts $\mathcal{D}_{\mathsf{test}}=\{T_j\}_{j=1}^n$, penalty $\Lambda$, committee size $K$
\State Compute the empirical disagreement minimizer $\pi_0 \in \arg\min_{\pi\in\Pi} \widehat d_M(\pi)$.
\State Let $q^{\star}\gets \textsc{RegularizedSparseValidatorDist}(\Pi,\pi_0,\mathcal{D}_{\mathsf{train}}, \mathcal{D}_{\mathsf{test}},\Lambda,K)$. 
\State Sample a single validator sequence $\Phi \sim q^\star$.
\State \Return the selective policy $(\pi_0, \tau_{\pi_0, \Phi})$.
\end{algorithmic}
\end{algorithm}

Our main result for this section is a sample complexity guarantee for \Cref{alg: agnostic seqrejectron}.

\begin{theorem}\label{thm:misspecified-deterministic-main}
Let $\delta>0$. Let $\pi_0 \in \arg\min_{\pi \in \Pi} \widehat{d}_M(\pi)$. Fix a penalty $\Lambda > 0$ and committee size $K \ge 1$, and let $q^{\star}\in \Delta(\Pi^K)$ be the equilibrium distribution from \Cref{lem:symmetrically-regularized-equilibrium}. Define the complexity measure $Z \coloneqq (K+1)\log|\Pi| + \log\tfrac{4}{\delta}$. Then, with probability at least $1-\delta$, the randomized selective policy $(\pi_0,\tau_{\pi_0,\Phi})$ with validator set drawn $\Phi \sim q^\star$ satisfies\begin{align*}
   &\mathbb{E}_{\Phi \sim q^\star} \big[ \alpha_M(\pi_0, \tau_{\pi_0,\Phi}) \big] \;\le\; (K+1)\Delta + \frac{1}{\Lambda} + 2(K+1)\sqrt{\frac{Z}{2m}}\\
   &\text{and} \quad \mathbb{E}_{\Phi \sim q^\star} \big[ \Regret_N(\pi_0,\tau_{\pi_0,\Phi};c) \big] \;\le\; \costmax\left( \Delta + \Lambda \Delta + \frac{1}{K} + 2\Lambda \sqrt{\frac{Z}{2m}} + \sqrt{\frac{Z}{2n}} \right).
\end{align*}
\end{theorem}

\begin{corollary}\label{cor: misspecified optimal rates}
There exists a choice of parameters $K = \Lambda$ (depending on $m, n, \Delta, |\Pi|$, and $\delta$) such that, with probability at least $1-\delta$, both $\mathbb{E}_{\Phi \sim q^\star}[\alpha_M(\pi_0, \tau_{\pi_0,\Phi})]$ and $\mathbb{E}_{\Phi \sim q^\star}[\Regret_N(\pi_0, \tau_{\pi_0,\Phi}; c)]/\costmax$ are bounded by $\widetilde{O}\left( \Delta^{1/2} + \left(\log|\Pi|/m\right)^{1/5} + \left(\log|\Pi|/n\right)^{1/3} \right)$.
\end{corollary}

\paragraph{Comparing to \Cref{alg: mdprejectron}.} \Cref{alg: agnostic seqrejectron} departs from the standard realizable $\mathsf{SeqRejectron}$ in two key ways. First, because exact consistency is impossible, we replace the hard version space with a penalty-based regularization. Second, and more subtly, we omit the derandomization step. In \Cref{alg: mdprejectron}, we draw $k$ independent validator sets and take their union to boost the target-side soundness to a high-probability guarantee. In the agnostic setting, however, taking the union of $k$ validator sets would linearly amplify the irreducible misspecification error (scaling the penalty to $O(k\Delta)$), which fundamentally degrades the regret bound. To prevent this error amplification, we sample only a single validator set $\Phi \sim q^\star$, meaning our resulting target-side soundness guarantees hold in expectation over this internal algorithmic randomness.

\subsection{Off-policy test trajectories under misspecification}

So far, we have assumed that the expert $\pi^*$ is the same across train and test environments. In practice, however, train and test demonstrations often come from different sources---for example, using one expert to pilot a simulated car, while using dashcam aggregated from multiple drivers. This creates a challenging off-policy gap where the training behavior $\pi^{\mathsf{tr}}$ and test behavior $\pi^{\mathsf{te}}$ differ, and neither necessarily lies in the policy class $\Pi$. Fortunately, the same idea we used for misspecified experts---finding a validator distribution via a \emph{regularized} game, which induces a stopping time---extends readily to policy shift between environments. 

Concretely, suppose $\mathcal D_{\mathsf{train}}$ is generated by a deterministic policy $\pi^{\mathsf{tr}}$, while $\mathcal D_{\mathsf{test}}$ is generated by a different deterministic policy $\pi^{\mathsf{te}}$, with neither policy necessarily in $\Pi$. Let $d_M^{\mathsf{tr}}$ and $d_N^{\mathsf{te}}$ denote the corresponding disagreement probabilities in $(M,\pi^{\mathsf{tr}})$ and $(N,\pi^{\mathsf{te}})$, respectively, and let $\Regret_N^{\mathsf{te}}$ denote regret against $\pi^{\mathsf{te}}$ in the test environment. Then, \Cref{alg: agnostic seqrejectron} has the following guarantee:

\begin{proposition}\label{cor:off-policy-misspecified-rates}
Let $\Delta_{\mathsf{off}} \coloneqq \min_{\bar\pi\in\Pi}\max\{d_M^{\mathsf{tr}}(\bar\pi), d_N^{\mathsf{te}}(\bar\pi)\}$.\footnote{Note that as long as there exists some policy $\pi \in \Pi$ which is able to behave similarly to $\pi^{\mathsf{tr}}$ on $M$ and to $\pi^\mathsf{te}$ on $N$, then the irreducible error $\Delta_\mathsf{off}$ will be small.} We can run \Cref{alg: agnostic seqrejectron} with some parameters $K = \Lambda$ (depending on $m, n, \Delta_{\mathsf{off}}, |\Pi|$, and $\delta$) such that, with probability $1-\delta$, both $\mathbb{E}_{\Phi \sim q^\star}[\alpha_M(\pi_0, \tau_{\pi_0,\Phi})]$ and $\mathbb{E}_{\Phi \sim q^\star}[\Regret_N^{\mathsf{te}}(\pi_0, \tau_{\pi_0,\Phi}; c)]/\costmax$ are $\widetilde{O}\left( \Delta_{\mathsf{off}}^{1/2} + \left(\log|\Pi|/m\right)^{1/5} + \left(\log|\Pi|/n\right)^{1/3} \right)$.
\end{proposition}

\section{Proof of Concept}
\label{sec:experiments}
We evaluate $\SeqRejectron$ in the LunarLander-v3 environment \citep{towers2024gymnasium}, with the goal of empirically characterizing the completeness-soundness tradeoff: the algorithm should stop minimally on train while reliably detecting out-of-distribution dynamics and stopping at test time.

\paragraph{Experimental setup.}
The step cost $c_h$ is a quadratic penalty on distance to the landing pad, velocity and angular velocity with a cap at 1. If the lander ever crashes, it incurs the maximum step-wise cost for the remainder of the episode. Costs are then normalized to a maximum episode cost of 1. For the dynamics shift, the source domain $M$ is windless and uses a reduced initial random impulse, while the target domain $N$ applies a rightward wind force. Thus the environment shift consists of the addition of wind and the stronger initial impulse. The expert $\pi^{\star}$ is a stochastic state-feedback controller, parameterized by a neural network on the raw LunarLander observation features, tuned to land in the target environment, but not necessarily optimal under the custom cost. The base learner $\pi_0$ is fit by MLE on labeled demonstrations $\mathcal{D}_{\mathsf{train}}$ drawn only from $M$.

To ensure tractability, we replace the game-theoretic construction of \Cref{section: implementation} with an adversarial posterior-sampling heuristic. First, using the labeled training demonstrations $\mathcal{D}_{\mathsf{train}}$, we form a Bayesian posterior over the policy parameters and sample a pool of candidates approximately consistent with $\pi^{\star}$ on $M$. Next, we evaluate this pool on the unlabeled test trajectories $\mathcal{D}_{\mathsf{test}}$ and greedily select the candidate that maximizes the cumulative squared Hellinger disagreement with the base learner $\pi_0$ (since our goal is to challenge our base learner on the test data). Repeating this selection $K = 3$ times yields a validator set $\Phi$, which defines a stopping rule as in \eqref{eq: stochastic stopping time}.

\begin{figure}[ht]
    \centering
    \begin{minipage}{0.35\textwidth}
        \centering
        \includegraphics[width=\linewidth]{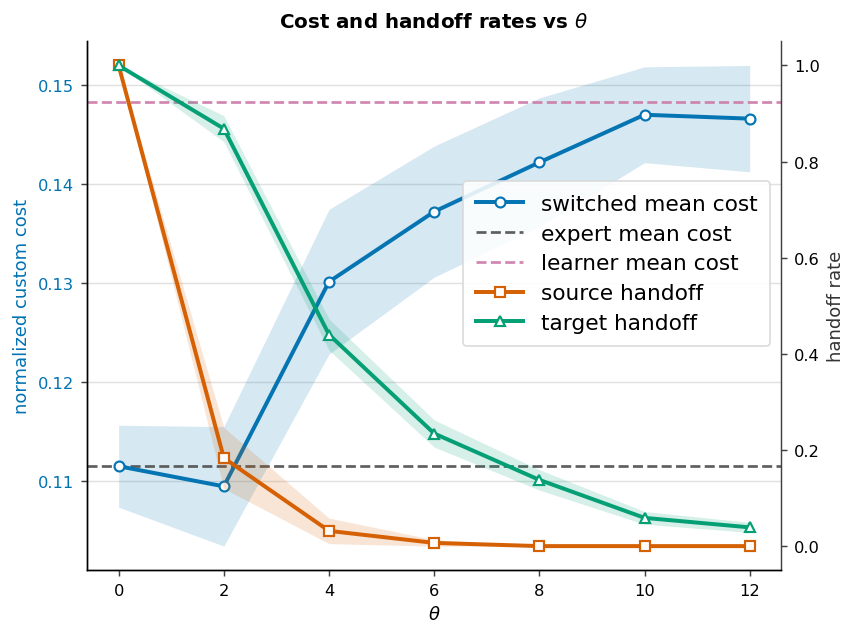}
    \end{minipage}
    \begin{minipage}{0.31\textwidth}
        \centering
        \includegraphics[width=\linewidth]{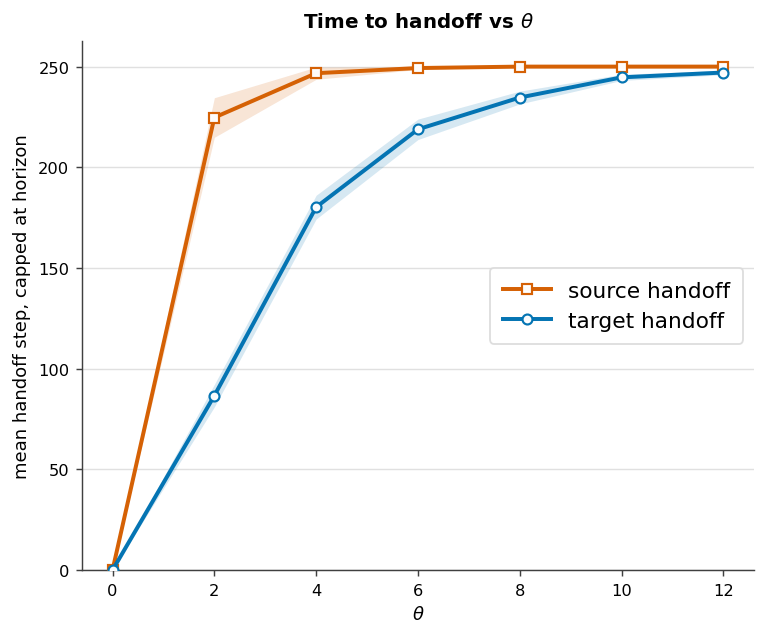}
    \end{minipage}
    \begin{minipage}{0.31\textwidth}
        \centering
        \includegraphics[width=\linewidth]{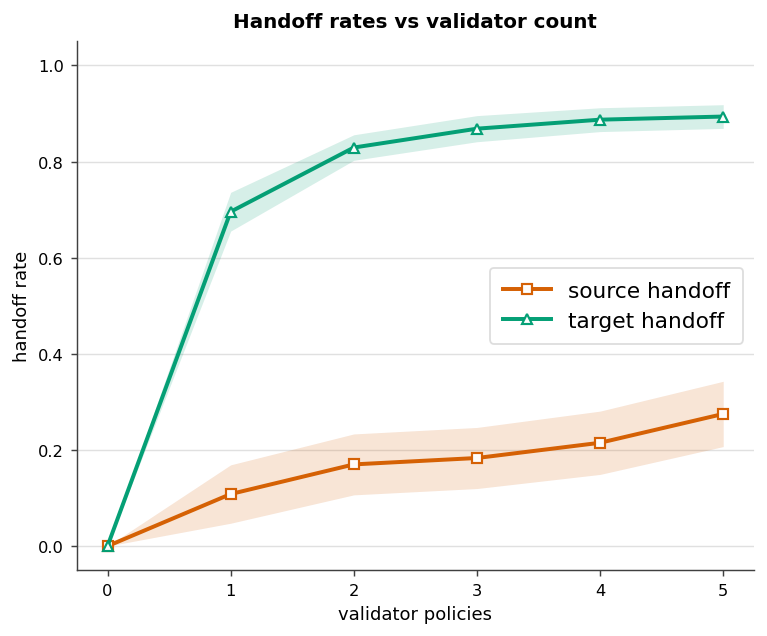}
    \end{minipage}
    \caption{Left: normalized target switched cost and source/target handoff rates as functions of the safety threshold $\theta$. Dashed horizontal lines show the expert and learner mean costs. Middle: mean handoff time as a function of $\theta$, with trajectories that never hand off capped at the episode horizon. Right: source and target handoff rates as a function of the number of validator policies, with $\theta = 2$ fixed. In all plots, shaded bands indicate standard error across trials.}
    \label{fig: crash and handoff}
\end{figure}

\paragraph{Results.} We evaluate the framework over $20$ independent trials. In each trial, we collect $m=30$ labeled training trajectories and $n=30$ unlabeled test trajectories over a range of $\theta$ threshold values. Figure~\ref{fig: crash and handoff} summarizes the resulting tradeoffs. In the left panel, the source handoff rate collapses steeply to near $0\%$, confirming that the validators rarely flag in-distribution trajectories. The target handoff rate decays far more gradually with the trajectory cost mirroring its behavior. As handoff becomes less frequent, the switched policy's cost approaches the learner's cost. The dashed black line represents the expert's average cost on the test environment. The red dashed line represents the expected cost of $\pi_0$, and is the cost that would have been obtained by running vanilla behavior cloning, without any abstention mechanism. This shows that $\SeqRejectron$ (with a suitable choice of $\theta = 2$) is able to reliably detect and abstain when it encounters environment shift.

The middle panel shows that larger $\theta$ delays handoff, as expected, since the validator allows more cumulative disagreement before switching. Finally, the right panel shows that increasing the number of validators raises handoff rates, particularly in the target environment, by making the stopping rule more sensitive to disagreement. This effect quickly saturates, suggesting that $\SeqRejectron$ is somewhat robust to the choice of validator set size.

\section{Conclusion}
\label{sec: conclusion}
We introduced selective imitation learning, a framework in which a learner may abstain mid-trajectory when its training data is uninformative about the test environment. Our algorithm, $\SeqRejectron$, certifies trajectory prefixes via a sparse set of validator policies, yielding horizon-free sample complexity for deterministic classes under sparse costs, and analogous guarantees for stochastic classes. We studied the asymmetric objective of completeness on training and soundness on test, and established an $\Omega(1/\epsilon^3)$ lower bound in the single-step setting that matches the labeled-sample cost-driven exponent. The framework extends to misspecified policy classes, with guarantees that degrade gracefully with the best-in-class misspecification.

Several directions remain open. On the computational side, the stochastic setting does not yet admit an oracle-efficient reduction analogous to the deterministic setting, and we leave the development of practical surrogates for the cumulative Hellinger objective to future work. On the statistical side, the variance-driven $\tilde{O}(\epsilon^{-5})$ labeled-sample rate for stochastic policies remains unmatched by our lower bound, leaving a gap between the cost-driven and variance-driven regimes. Finally, while our framework treats the test dynamics as entirely unknown, it would be natural to incorporate partial structural knowledge---such as a prior or bounded TV distance, density ratio, or parametric shift---to tighten the stopping rate on the test side. We hope this work establishes abstention as a principled and theoretically grounded response to unanticipated dynamics shift.

\subsection*{Acknowledgments}
The authors gratefully acknowledge support from NSF Award CCF:AF:2504016 and a Schmidt Sciences AI2050 Early Career Fellowship.

\bibliography{references}

\newpage
\appendix
\tableofcontents
\newpage

\section{Additional Related Work}
\label{sec: additional related work}
\subsection{Background on imitation learning and reinforcement learning}

\paragraph{Reinforcement learning.} The standard framework for sequential decision-making is the Markov decision process (MDP)~\citep{sutton1998reinforcement}, where an agent maximizes cumulative reward by interacting with an environment governed by unknown transition dynamics. In the offline setting, where the agent must learn from a fixed dataset without further environment interaction, pessimism-based approaches~\citep{rashidinejad2021bridging, kumar2020conservative, levine2020offline} provide guarantees when the dataset covers the target policy. Our work operates in a similarly offline regime, but with a key difference: we observe datasets from \emph{two} environments and provide PAC-style guarantees without any online interaction with the test environment.

\paragraph{Imitation learning.} The dominant paradigm for learning from demonstrations is \emph{behavior cloning} (BC), which reduces imitation learning to supervised prediction~\citep{pomerleau1988alvinn, ross2010efficient}. A central limitation of BC is \emph{compounding error}: small per-step mistakes cause the learner to drift to states outside the expert's distribution, where further errors compound. Interactive methods address this by querying the expert on the learner's own rollouts: DAgger~\citep{ross2011reduction} reduces regret to $O(\epsilon H)$ by collecting on-policy corrections, and subsequent work gave game-theoretic formulations~\citep{swamy2021moments} and adversarial approaches via occupancy-measure matching~\citep{ho2016generative, ziebart2008maximum}. Imitation from observation~\citep{torabi2018behavioral} studies the special case where only state sequences (not actions) are available, which is similar to our assumption of unlabeled test trajectories, but does not provide theoretical guarantees.

\paragraph{Sample complexity of imitation learning.} A precise theoretical understanding of BC's sample complexity has emerged over the past several years. \citet{rajaraman2020toward} established fundamental limits: offline BC requires $\Omega(H^2/\epsilon)$ trajectories in the worst case, while online methods can achieve $O(H/\epsilon)$, and \citet{swamy2022minimax} showed that replay estimation achieves minimax-optimal online IL by reducing empirical variance in the expert visitation distribution. This apparent gap between offline and online IL was recently revisited by \citet{foster2024behavior}, who showed that under weight-sharing, log-loss BC achieves \emph{horizon-independent} sample complexity via a trajectory-level Hellinger analysis, whenever cumulative costs are controlled and the policy class shares parameters across time steps; this matches the minimax rate even among interactive algorithms. Our work builds directly on this Hellinger-based analysis, extending it to the environment-shift setting where test dynamics may differ arbitrarily from training, and replacing the standard regret objective with a stopped regret that accounts for the possibility of abstention.

\subsection{Background on selective classification and learning with distribution shift}

\paragraph{Learning under distribution shift.} The classical approach to distribution shift is \emph{covariate shift adaptation}, where the marginal input distribution changes but the labeling rule is preserved. \citet{shimodaira2000improving} introduced importance-weighting as a principled correction, and \citet{sugiyama2007covariate} gave an importance-weighted cross-validation procedure for model selection under this shift. On the theoretical side, \citet{ben2010theory} and \citet{mansour2009domain} bound target error in terms of source error and a divergence between domains. However, \citet{ben2012hardness} establish that without further assumptions on the relationship between source and target---such as bounded density ratios or overlapping supports---successful domain adaptation is impossible in general. These hardness results make clear that any distribution-free guarantee under arbitrary shift must either make structural assumptions or give up on committing to an action everywhere.

\paragraph{Selective classification and PQ learning.} \emph{Selective classification}~\citep{chow2003optimum, el2010foundations, geifman2017selective} offers a way out: by allowing a classifier to abstain on inputs it cannot reliably label, one can trade coverage for accuracy. Recent work on partial deferral for sequence prediction studies one-time handoff to an expert after a chosen point in the sequence, which is closely related in spirit to our expert-handoff view of abstention \citep{rayan2025learningpartiallydefersequences}; however, they do not consider environment shift, which is the central focus of this work. Closest to our formal setup, \citet{goldwasser2020beyond} introduced the \emph{PQ learning} framework, showing that abstention is precisely the mechanism that circumvents the hardness of arbitrary distribution shift. Given labeled samples from $P$ and unlabeled samples from $Q$, PQ learning asks for a selective classifier with low error on $Q$ and low rejection rate on $P$---with no assumptions on the relationship between $P$ and $Q$. Their algorithm, $\mathsf{Rejectron}$, iteratively builds validator classifiers that agree on $P$ but disagree on $Q$, and abstains wherever such disagreement exists. Subsequent work~\citep{kalai2021efficient} showed that the computational complexity lies in between PAC learning and Agnostic PAC learning. Another closely related model is the Testable Distribution Shift (TDS) \cite{klivans2023testable} which allows for rejecting the entire test distribution instead of selectively predicting. There has been a flurry of recent work designing computationally efficient algorithms for these models under assumptions only on the train distribution ~\citep{klivans2024learning,chandrasekaran2024efficient,goel2024tolerant}. Our formulation can be viewed as a sequential analogue of PQ learning (\Cref{sec: problem setup}), and the validator-based construction of $\SeqRejectron$ can be thought of as a generalization of $\mathsf{Rejectron}$ to trajectory-level certification.

\subsection{Robustness, uncertainty quantification, and domain adaptation in RL}

\paragraph{Sim-to-real transfer and domain randomization.}
The \emph{reality gap}---the mismatch between simulated and real-world appearance, physics, and dynamics---is a central obstacle to deploying learned policies on physical systems~\citep{tobin2017domain}. 

The dominant mitigation is \emph{domain randomization}: training over a broad distribution of simulated environments so the real world falls within the training support. This has been applied to visual parameters~\citep{tobin2017domain, sadeghi2016cad2rl, james2019sim}, physical parameters such as friction, mass, and actuator delays~\citep{peng2018sim, tan2018sim}, and scaled to dexterous manipulation and Rubik's-cube solving~\citep{andrychowicz2020learning, akkaya2019solving}. Complementing these empirical milestones, recent literature has formalized the theoretical foundations of this approach. Specifically, these works have bounded the sim-to-real gap for uniform domain randomization~\citep{chen2021understanding}, established the method's optimal asymptotic sample efficiency~\citep{fujinami2025domain} for linear systems, and proved the statistical consistency of data-driven, offline domain randomization~\citep{fickingerstatistical}

A complementary line of work closes the gap \emph{adaptively}: \cite{chebotar2019closing} iteratively refine the randomization distribution using real rollouts; \cite{yu2017preparing} and \cite{kumar2021rma} condition policies on online system-identification modules that infer dynamics at deployment; \cite{rusu2017sim} use progressive networks for visual transfer; and \cite{mehta2020active} steer randomization toward maximally informative environments. \cite{codevilla2019exploring} document the resulting failure modes of behavior cloning under sim-to-real shift, motivating principled responses to out-of-distribution conditions. 

All of these methods attempt to \emph{bridge} the reality gap. Our work is complementary: rather than correcting for an unknown shift, we learn to \emph{detect} when dynamics have changed too much for reliable execution and abstain, with distribution-free PAC guarantees that require no assumptions on the magnitude or structure of the mismatch.

\paragraph{Robust reinforcement learning.} Robust RL seeks policies that perform well under worst-case dynamics drawn from a prespecified uncertainty set. The foundational framework is the robust MDP, in which nature adversarially selects transitions within a rectangular uncertainty set; \cite{iyengar2005robust} and \cite{nilim2005robust} independently showed that such problems admit tractable dynamic-programming solutions, and \cite{wiesemann2013robust} extended this to more general uncertainty sets and ambiguity models. In the deep RL setting, \cite{pinto2017robust} train a protagonist against an adversary that applies destabilizing forces (RARL), and \cite{rajeswaran2016epopt} optimize over an ensemble of source models (EPOpt) to guard against model error. \cite{zhang2020robust} study robustness to observation and action perturbations in deep RL, and \cite{tessler2019action} isolate the action-perturbation setting. On the theoretical side, \cite{panaganti2022sample} and \cite{shi2024distributionally} establish sample complexity bounds for learning robust policies from offline data under $(\mathbf{s},\mathbf{a})$-rectangular and distributionally robust uncertainty sets, respectively. All of these methods require the learner to commit ahead of time to an uncertainty set or dynamics prior---implicitly assuming knowledge of the shift's structure or magnitude.

\paragraph{Imitation learning under dynamics mismatch.} Classical IL theory assumes that the learner and expert share the same transition dynamics. A growing body of work relaxes this assumption.

Several approaches use adversarial methods to match occupancy measures across environments: \cite{ho2016generative} introduce GAIL, which frames IL as a distribution-matching problem and provides a foundation for dynamics-robust extensions; \cite{fu2017learning} show that adversarial IRL (AIRL) can recover reward functions that transfer under moderate dynamics shift; and \cite{kostrikov2019imitation} propose a distribution-correction method (DAC) that compensates for off-policy mismatch via density-ratio reweighting.

Other works directly target the setting where expert and learner operate under different dynamics. \cite{gangwani2020state} and \cite{desai2020imitation} use adversarial state-distribution matching and transition-dynamics alignment respectively. \cite{viano2021robust} bound performance degradation as a function of the $\ell_1$ distance between transition kernels, and \cite{chae2022robust} train policies that are robust across a family of perturbed MDPs (RIME). \cite{eysenbach2020off} compensate for dynamics mismatch by modifying the reward with a classifier that distinguishes source from target transitions, though this requires online RL access to the target domain. \cite{lauffer2025imitation} address the covariate shift that arises when a learned policy diverges into unfamiliar dynamics, using on-policy expert corrections for multi-turn LM agents. Finally, \cite{espinosa2025efficient} study the misspecification setting in which the expert falls outside the learner's policy class.

All of these methods attempt to \emph{correct for} or \emph{tolerate} dynamics mismatch via reward modification, distribution matching, or robust optimization, and most require either interactive access to the target environment or assumptions on the magnitude of the shift. Our work is complementary: we make no assumptions on the dynamics and instead learn, from offline data alone, to \emph{detect} when the shift is too large to act reliably and abstain, with distribution-free PAC guarantees on both the stopping rate and the stopped regret.

\paragraph{Control-theoretic safety and reachability.} Classical approaches to safe deployment ground safety guarantees in known system dynamics. Hamilton-Jacobi reachability~\citep{tomlin2000game, bansal2017hamilton} computes the exact set of states from which a safety constraint can be guaranteed, and control barrier functions (CBFs)~\citep{ames2019control} enforce forward invariance of a safe set via Lyapunov-style certificates. Prajna and Jadbabaie~\citep{prajna2004safety} give analogous barrier certificates for hybrid systems. These methods offer strong formal guarantees, but require an accurate model of the transition dynamics.

\paragraph{Uncertainty-aware deployment.}
A natural response to distribution shift is to monitor epistemic uncertainty at deployment and abstain or intervene when it is high; MC dropout~\citep{gal2016dropout} and deep ensembles~\citep{lakshminarayanan2017simple} are the dominant practical tools for obtaining such estimates. These estimates have been deployed directly in safety-critical settings: \cite{lutjens2019safe} use MC dropout and bootstrapped ensembles to modulate collision-avoidance policies when epistemic uncertainty is high, and \cite{kahn2017uncertainty} propagate learned uncertainty estimates through a model-based planner for robot navigation. However, the thresholds used to trigger abstention in these methods are typically heuristic. \cite{neggatu2025evaluation} use uncertainty estimates to switch between an offline RL policy and a more conservative behavior-cloning policy, but do not prove theoretical guarantees, unlike ours. Several works \citep{haider2023out, prashant2025guaranteeing} frame out-of-distribution shift in RL as changes in the transition dynamics, but the proposed methods do not rigorously bound the downstream regret.

Conformal prediction (see \cite{angelopoulos2023conformal} for a foundational treatment) offers a distribution-free alternative with rigorous finite-sample coverage guarantees. \cite{tibshirani2019conformal} extend the framework to covariate shift via importance weighting; and \cite{gibbs2021adaptive} develop online conformal methods that adapt to distribution shift as it occurs. Recent work has extended these ideas to sequential and dynamical settings. \cite{lindemann2023safe} use conformal prediction to construct probabilistic prediction regions for safe motion planning in dynamic environments; \cite{dixit2023adaptive} apply adaptive conformal prediction to quantify multi-step-ahead uncertainty for MPC among dynamic agents; \cite{lindemann2023conformal} give conformal runtime verification guarantees for signal temporal logic specifications; and \cite{lee2024single} establish conformal coverage guarantees from a single trajectory of temporally correlated data generated by an unknown stochastic dynamical system. 

All of these methods share our high-level goal of detecting when a deployed agent has ventured outside its training support and should stop. Ensemble and conformal methods differ from ours in a fundamental way, however: ensemble approaches rely on heuristic variance thresholds to trigger abstention, while conformal methods provide rigorous marginal coverage for \emph{prediction sets} at each step. Neither directly bounds the regret accumulated along a trajectory before abstention. Our approach avoids both heuristics and prediction sets, instead providing frequentist PAC-style guarantees directly on the stopping time that bound the stopped regret in the test environment.

\section{Additional Results and Proofs from the Deterministic Section}
\subsection{Supporting technical lemmas}

\begin{lemma}[Prefix coupling for deterministic policies]
\label{lemma:prefix-coupling}
Let $\pi$ and $\pi^{\star}$ be deterministic policies on a horizon-$H$ MDP with pre-action filtration $\mathcal{G}_h = \sigma(s_1, a_1, \dots, s_{h-1}, a_{h-1}, s_h)$. Define the first-deviation time
\[
    \tau_{\pi, \{\pi^{\star}\}} \coloneqq \min\bigl\{h \in [H] : \pi_h(s_h) \neq \pi^{\star}_h(s_h)\bigr\},
\]
with the convention $\min\emptyset = H+1$. Let $\tau$ be any stopping time adapted to $\{\mathcal{G}_h\}$, and define $h^{\star} \coloneqq \min(\tau, \tau_{\pi, \{\pi^{\star}\}})$, and define the stopped sigma-algebra
\[
    \mathcal{G}_{h^{\star}} \coloneqq \bigl\{A \subseteq \Omega : A \cap \{h^{\star} = h\} \in \mathcal{G}_h \text{ for every } h \in [H+1]\bigr\}.
\]
Intuitively, $\calG_{h^{\star}}$ represents the information available up until the (random) time $h^{\star}$. Then, for any environment $M$, every event $\mathcal{E} \in \mathcal{G}_{h^{\star}}$ satisfies $\Pr_{M,\pi}[\mathcal{E}] = \Pr_{M,\pi^{\star}}[\mathcal{E}]$. Consequently, for every bounded $\mathcal{G}_{h^{\star}}$-measurable random variable $Z$, we also have $\E_{M,\pi}[Z] = \E_{M,\pi^{\star}}[Z]$.
\end{lemma}

\begin{proof}
We first check that $h^{\star}$ is a stopping time. For each $k \in [H]$, the event $\{\tau_{\pi, \{\pi^{\star}\}} \le k\} = \bigcup_{j=1}^{k}\{\pi_j(s_j) \neq \pi^{\star}_j(s_j)\}$ belongs to $\mathcal{G}_k$ because each $s_j$ is $\mathcal{G}_j \subseteq \mathcal{G}_k$-measurable and $\pi_j, \pi^{\star}_j$ are deterministic; for $k=H+1$ the claim is trivial. Hence $\tau_{\pi, \{\pi^{\star}\}}$ is a stopping time, and so is $h^{\star} = \min(\tau, \tau_{\pi, \{\pi^{\star}\}})$.

It remains to show that $\pi$ and $\pi^{\star}$ induce the same distribution on $\mathcal{G}_{h^{\star}}$-measurable events. Fix $k \in [H]$ and any pre-action prefix $\mathbf{t}_k = (s_1, a_1, \dots, s_{k-1}, a_{k-1}, s_k)$ with $h^{\star}(\mathbf{t}_k) = k$. On the event $\{h^{\star} = k\}$ we have $k \le \tau_{\pi, \{\pi^{\star}\}}$, so $\pi_j(s_j) = \pi^{\star}_j(s_j)$ for all $j < k$, and therefore\begin{align*}
    \Pr_{M,\pi}[(s_1, a_1, \dots, s_{k-1}, a_{k-1}, s_k) = \mathbf{t}_k,\, h^{\star} = k] &= P_0(s_1)\prod_{j=1}^{k-1} P_j(s_{j+1} \mid s_j, \pi_j(s_j)) \\
    &= \Pr_{M,\pi^{\star}}[(s_1, a_1, \dots, s_{k-1}, a_{k-1}, s_k) = \mathbf{t}_k,\, h^{\star} = k].
\end{align*}
When $k = H+1$, the same argument applied to full trajectories shows that the probabilities also agree on the slice $\{h^{\star} = H+1\}$, since $\pi_j(s_j) = \pi^{\star}_j(s_j)$ for every $j \in [H]$ there. For any $\mathcal{E} \in \mathcal{G}_{h^{\star}}$, the slice $\mathcal{E} \cap \{h^{\star} = k\}$ is $\mathcal{G}_k$-measurable, so its probability is a sum of prefix probabilities of the form above; summing over $k$ gives $\Pr_{M,\pi}[\mathcal{E}] = \Pr_{M,\pi^{\star}}[\mathcal{E}]$. Because the two laws are equal on $\mathcal{G}_{h^{\star}}$, their expectations agree for every bounded $\mathcal{G}_{h^{\star}}$-measurable random variable.
\end{proof}

\begin{lemma}[General statement of \Cref{lemma:sparse-valid-validator-distribution}]\label{lemma: generalized sparse selective policies}
    Let $M$ be an positive integer and let $\mathcal{T} \subseteq [M]^n$ be any set of vectors. For any target $\epsilon \in (0, 1)$, set $K = \lceil 1/\epsilon \rceil$. There exists a probability distribution $q^{\star}$ over subsets $S \subseteq \mathcal{T}$ of size at most $K$ with the following two properties.
    \begin{enumerate}
        \item \textbf{(Coverage)} For every $\tau \in \mathcal{T}$:
        \[
            \E_{S \sim q^{\star}}\left[\frac{1}{n}\sum_{i=1}^{n}
            \mathbf{1}\left[\min_{v \in S} v_i > \tau_i\right]\right] < \epsilon.
        \]
        \item \textbf{(Sparsity)} The support of $q^{\star}$ consists entirely of sets of size at most $\lceil 1/\epsilon \rceil$.
    \end{enumerate}
\end{lemma}

\begin{proof}
    Define a finite zero-sum game in which the maximizer chooses $p \in \Delta(\mathcal{T})$ and the minimizer chooses $q \in \Delta(\mathcal{T}^K)$, with payoff
    \[
        \phi(p, q) := \E_{\tau \sim p}\, \E_{S \sim q}\left[
        \frac{1}{n}\sum_{i=1}^{n}
        \mathbf{1}\left[\min_{v \in S} v_i > \tau_i\right]
        \right].
    \]
    Since $\phi$ is bilinear in $(p, q)$, von Neumann's minimax theorem gives
    \[
        \min_{q}\max_{p}\; \phi(p,q) = \max_{p}\min_{q}\; \phi(p,q).
    \]

    To bound the maximin, fix any $p \in \Delta(\mathcal{T})$ and let $p^K$ denote the strategy of forming $S = \{v_1, \dots, v_K\}$ by drawing $v_1, \dots, v_K$ i.i.d.\ from $p$. Then
    \begin{align*}
        \min_{q}\; \phi(p, q) \leq \phi(p, p^K)
        = \frac{1}{n}\sum_{i=1}^{n}
          \Pr_{\tau,\, v_1,\dots,v_K \sim p}\left(
          \min_{j \in [K]} (v_j)_i > \tau_i
          \right).
    \end{align*}
    For a fixed coordinate $i$, the event $\{\min_j (v_j)_i > \tau_i\}$ holds if and only if $\tau_i$ is strictly the smallest value among the $K+1$ i.i.d.\ draws $\tau_i, (v_1)_i, \dots, (v_K)_i$. Since these draws are exchangeable, each is equally likely to be the strict minimum, and since these $K+1$ events are mutually exclusive, each occurs with probability at most $1/(K+1)$. Therefore
    \[
        \phi(p, p^K) \leq \frac{1}{n}\sum_{i=1}^{n}\frac{1}{K+1} = \frac{1}{K+1}.
    \]
    Since this holds for every $p \in \Delta(\mathcal{T})$, we conclude $\max_{p}\min_{q}\, \phi(p, q) \leq 1/(K+1)$.

    Combining with the minimax equality and using $K = \lceil 1/\epsilon \rceil$ so that $K + 1 > 1/\epsilon$:
    \[
        \min_{q}\max_{p}\; \phi(p, q) = \max_{p}\min_{q}\; \phi(p, q) \leq \frac{1}{K+1} < \epsilon.
    \]
    Hence there exists $q^{\star} \in \Delta(\mathcal{T}^K)$ with $\max_{p}\, \phi(p, q^{\star}) < \epsilon$. Taking $p = \delta_\tau$ for any fixed $\tau \in \mathcal{T}$ yields the coverage condition. The sparsity condition holds because every $S$ in the support of $q^{\star}$ satisfies $|S| \leq K = \lceil 1/\epsilon \rceil$.
\end{proof}

\subsection{Analysis of stepwise reduction to PQ-learning}\label{subsec: per step pq}

Assume $\calA = \{0,1\}$ and $\pi^{\star}\in \Pi$ is deterministic. For each $h \in [H]$, let
$\mu_h^M$ and $\mu_h^N$ denote the expert's step-$h$ state marginals under $M$ and $N$. We apply
Rejectron~\citep{goldwasser2020beyond} at each step $h$ with labeled examples $(s, \pi_h^\star(s))$
for $s \sim \mu_h^M$ and unlabeled examples from $\bar\mu_h \coloneqq \frac{1}{2}\mu_h^M +
\frac{1}{2}\mu_h^N$, yielding a predictor-selector pair $(\tilde\pi_h, g_h)$. Define $\tilde\pi =
(\tilde\pi_1,\dots,\tilde\pi_H)$, $g = (g_1,\dots,g_H)$, and $\tau_g = \min\{h : g_h(s_h) = 0\}$.
We run each step-$h$ learner with confidence $\delta/H$ and union-bound over $h$, so all
guarantees below hold simultaneously with probability at least $1-\delta$.

For tolerance $\gamma$, the PQ guarantee gives for every $h \in [H]$:
\begin{align}
\Pr_{s\sim\mu_h^M}[g_h(s)=0] &\le \gamma, \label{eq:pq-train-reject}\\
\Pr_{s\sim\bar\mu_h}[g_h(s)=1,\, \tilde\pi_h(s)\neq\pi_h^\star(s)] &\le \gamma.
\label{eq:pq-mix-error}
\end{align}
Since $\bar\mu_h$ averages $\mu_h^M$ and $\mu_h^N$, both marginals inherit a $2\gamma$ accepted-error
bound from \eqref{eq:pq-mix-error}.

\paragraph{Abstention on $M$.} Let $E_h$ be the event that step $h$ is the \emph{first} time the learner either abstains or makes an unrejected mistake and let
$
A_h:=\bigcap_{t=1}^{h-1}\Bigl\{g_t(s_t)=1,\ \tilde{\pi}_t(s_t)=\pi_t^\star(s_t)\Bigr\}.
$
Thus \(A_h\) is the event that, before step \(h\), the learner has neither abstained nor made an unrejected mistake. By definition,
\[
E_h
=
A_h\cap
\Bigl(
\{g_h(s_h)=0\}
\cup
\{g_h(s_h)=1,\ \tilde{\pi}_h(s_h)\neq \pi_h^\star(s_h)\}
\Bigr).
\]
Since the event inside the intersection is \(\mathcal G_h\)-measurable, by \Cref{lemma:prefix-coupling}, we may change measure from \(\Pr_{M,\tilde{\pi}}\) to \(\Pr_{M,\pi^\star}\):
\begin{align*}
    \Pr_{M,\tilde{\pi}}(E_h)&=
\Pr_{M,\pi^\star}\!\left(
A_h\cap
\Bigl(
\{g_h(s_h)=0\}
\cup
\{g_h(s_h)=1,\ \tilde{\pi}_h(s_h)\neq \pi_h^\star(s_h)\}
\Bigr)
\right)\\
&\le
\Pr_{M,\pi^\star}\!\Bigl(
\{g_h(s_h)=0\}
\cup
\{g_h(s_h)=1,\ \tilde{\pi}_h(s_h)\neq \pi_h^\star(s_h)\}
\Bigr)\\
&\le
\Pr_{M,\pi^\star}\!\bigl(g_h(s_h)=0\bigr)
+
\Pr_{M,\pi^\star}\!\bigl(g_h(s_h)=1,\ \tilde{\pi}_h(s_h)\neq \pi_h^\star(s_h)\bigr)\\
&=\Pr_{s\sim \mu_h^M}\!\bigl[g_h(s)=0\bigr]
+
\Pr_{s\sim \mu_h^M}\bigl[g_h(s)=1, \tilde{\pi}_h(s)\neq \pi_h^\star(s)\bigr]\le 3\gamma.
\end{align*}
Since abstaining implies $E_h$ occurs for some $h$, we have
\[
    \alpha_M(\tilde\pi, g) \le \sum_{h=1}^H \Pr(E_h) \le 3H\gamma.
\]

\paragraph{Regret on $N$.}
Let \(F_h\) be the event that the first unrejected misstep occurs at step \(h\) prior to any abstention. Similar to the argument above, if
$
A_h:=\bigcap_{t=1}^{h-1}\{g_t(s_t)=1,\ \tilde{\pi}_t(s_t)=\pi_t^\star(s_t)\},
$
then
\[
F_h
=
A_h\cap
\{g_h(s_h)=1,\ \tilde{\pi}_h(s_h)\neq \pi_h^\star(s_h)\}.
\]
We may first change measure from \(\Pr_{N,\tilde\pi}\) to \(\Pr_{N,\pi^\star}\) while retaining the factor \(A_h\), and then drop \(A_h\). Thus
\begin{align*}
\Pr_{N,\tilde\pi}(F_h)
&\le
\Pr_{N,\pi^\star}\!\bigl(g_h(s_h)=1,\ \tilde{\pi}_h(s_h)\neq \pi_h^\star(s_h)\bigr) =
\Pr_{s\sim \mu_h^N}\!\bigl[g_h(s)=1,\ \tilde{\pi}_h(s)\neq \pi_h^\star(s)\bigr]\le 2\gamma.
\end{align*}
Since the learner incurs zero stopped regret unless an unrejected misstep occurs—which costs at most \(\costmax\)—we obtain

\[
    \Regret_N(\tilde\pi, g; c) \le \costmax \sum_{h=1}^H \Pr(F_h) \le 2\costmax H\gamma.
\]

\paragraph{Sample complexity.} Setting $\gamma = \Theta(\epsilon / (H\costmax))$ to achieve
$\Regret_N \le \epsilon$, and using the PQ sample complexity of $\tilde{O}(\log|\Pi|/\gamma^2)$
per step, yields $m = n = \tilde{O}\left(\frac{H^2 \costmax^2 \log|\Pi|}{\epsilon^2}\right)$ trajectories (since one trajectory provides an example for each step). This $H^2$ penalty is unavoidable in this approach: the reduction certifies $H$ isolated decisions rather than a single trajectory prefix, and the horizon factor enters through the per-step tolerance regardless of the cost structure.

\subsection{Proof of \Cref{lemma:sparse-valid-validator-distribution}}

Let $M = H+1$. For each policy $\pi \in \Pi_{\mathsf{version}}$, let $v^\pi \in [M]^n$ be its empirical stopping-time vector on the $n$ test trajectories, where the $j$-th coordinate is given by $(v^\pi)_j = \tau_{\pi_0, \{\pi\}}(T_j)$. Let $\mathcal{T} = \{v^\pi : \pi \in \Pi_{\mathsf{version}}\}$ be the finite set of all such vectors.

Applying \Cref{lemma: generalized sparse selective policies} to the set $\mathcal{T}$ with target error $\epsilon = \rho$ implies the existence of a distribution $q^{\star}$ over subsets $\Phi \subseteq \Pi_{\mathsf{version}}$ satisfying:
\begin{itemize}
    \item For all $\pi \in \Pi_{\mathsf{version}}$:
    $\E_{\Phi \sim q^{\star}}\left[\frac{1}{n}\sum_{j=1}^n
    \mathbf{1}[\tau_{\pi_0, \Phi}(T_j) > \tau_{\pi_0, \{\pi\}}(T_j)]\right]
    < \rho$.
    \item For all $\Phi \in \mathsf{supp}(q^{\star})$: $|\Phi| \leq \lceil 1/\rho \rceil$.
\end{itemize}
This is exactly the $q^{\star}$ required by \Cref{lemma:sparse-valid-validator-distribution}.

\subsection{Proof of \Cref{theorem: deterministic rejection sample complexity}}\label{subsec: proof of deterministic rejection sample complexity}

We prove the theorem through four steps: constructing the validator distribution, bounding the empirical risk, generalizing to the population, and bounding the regret on $N$.

Throughout, let $K = \lceil \log_2(5/\delta) \rceil \lceil 2/\eta \rceil$ denote the maximum ensemble size and\[
    Z\coloneqq (K+1)\log|\Pi| + \log\frac{5}{\delta}
\]
the complexity term appearing in generalization bounds.

\paragraph{Step 1 (Existence of $q^{\star}$).} By \Cref{prop:validator-noregret}, \Cref{alg:sparse-validator-dist} yields a $(\eta/2 + \xi)$-valid distribution $q^\star$ over subsets of $\Pi_{\mathsf{version}}$ such that every $\Phi' \in \mathsf{supp}(q^\star)$ satisfies $|\Phi'| \le \left\lceil \frac{2}{\eta}\right\rceil$ with probability $1-\delta/5$.

\paragraph{Step 2 (Bound the empirical risk).} The output $\Phi = \bigcup_{i=1}^k \Phi_i$ is the union of $k = \lceil \log_2(5/\delta) \rceil$ independent draws from $q^{\star}$, so $|\Phi| \leq k \cdot \lceil 2/\eta \rceil = K$ deterministically. By Markov's inequality applied to the $\eta/2 + \xi$ coverage guarantee,\begin{align*}
     &\Pr_{\Phi_i \sim q^{\star}}\left(
    \frac{1}{n}\sum_{j=1}^n \mathbf{1}[\tau_{\pi_0,\Phi_i}(T_j) >
    \tau_{\pi_0,\{\pi^{\star}\}}(T_j)] > \eta + 2\xi\right)\\
    &\leq \frac{\E_{\Phi_i \sim q^{\star}}\left[\frac{1}{n}\sum_{j=1}^n
    \mathbf{1}[\tau_{\pi_0,\Phi_i}(T_j) > \tau_{\pi_0,\{\pi^{\star}\}}(T_j)]\right]}{\eta + 2\xi}
    < \frac{\eta/2 + \xi}{\eta  + 2\xi} = \frac{1}{2}.
\end{align*}
We refer to the term on the left hand side as the \emph{risk}. Since $\tau_{\pi_0,\Phi}(T_j) = \min_i \tau_{\pi_0,\Phi_i}(T_j)$, then if any $\Phi_i$ achieves risk $\leq \eta$, then so does $\Phi$. Equivalently, $\Phi$ only has high risk if all $k$ independent draws have high risk. Since these draws are independent, \[
    \Pr\left(\frac{1}{n}\sum_{j=1}^n \mathbf{1}[\tau_{\pi_0,\Phi}(T_j) >
    \tau_{\pi_0, \{\pi^{\star}\}}(T_j)] > \eta +2\xi\right)
    \leq \prod_{i=1}^k \frac{1}{2}
    \leq \left(\frac{1}{2}\right)^{\lceil \log_2(5/\delta) \rceil}
    \leq \delta/5.
\]
For the training data, all policies in $\Phi \cup \{\pi_0\}$ belong to $\Pi_{\mathsf{version}}$ by construction, meaning they agree with $\pi^{\star}$ on every observed state-action pair in $\mathcal{D}_{\mathsf{train}}$. Therefore, for all $i \in [m]$:\[
    \tau_{\pi_0,\Phi}(S_i) = \tau_{\pi_0,\{\pi^{\star}\}}(S_i) = H+1,
\]
where the first equality holds because no policy in $\Phi$ ever disagrees with $\pi_0$ on any training trajectory (since both agree with $\pi^{\star}$ on all observed states), and the second holds because $\pi^{\star}$ and $\pi_0$ both belong to $\Pi_{\mathsf{version}}$ and hence agree on all observed actions. Since $\tau_{\pi_0,\Phi}(S_i) = H+1 > H$ for all $i$, the empirical stopping rate $\frac{1}{m}\sum_{i=1}^m \mathbf{1}[\tau_{\pi_0,\Phi}(S_i) \leq H] = 0$, and since $\tau_{\pi_0,\Phi}(S_i) =\tau_{\pi_0,\{\pi^{\star}\}}(S_i)$ for all $i$, the empirical risk $\frac{1}{m}\sum_{i=1}^m \mathbf{1}[\tau_{\pi_0,\Phi}(S_i) > \tau_{\pi_0,\{\pi^{\star}\}}(S_i)] = 0$ as well.

\paragraph{Step 3 (Bound generalization error on the expert's distribution).}
We use uniform convergence over $\mathcal{H}_K = \{(\pi_0, \Phi) : \pi_0 \in \Pi,\ \Phi \subseteq \Pi, |\Phi| \leq K\}$, which satisfies $|\mathcal{H}_K| \leq |\Pi|^{K+1}$ as $\pi_0$ and each of the at most $K$ elements of $\Phi$ are chosen from $\Pi$. We identify three failure events each occurring with probability at most $\delta/5$.

\textit{Test risk generalization.} Let $Y(\pi_0, \Phi) = \mathbf{1}[\tau_{\pi_0, \Phi}(T) > \tau_{\pi_0, \{\pi^{\star}\}}(T)]$, so that $\hat{\E}[Y] = \frac{1}{n}\sum_{j=1}^n Y(T_j)$ is the empirical test risk over $\mathcal{D}_{\mathsf{test}}$ and $\E_N^{\pi^{\star}}[Y] = \Pr_{N,\pi^{\star}}[\tau_{\pi_0, \Phi}(T) > \tau_{\pi_0, \{\pi^{\star}\}}(T)]$ is the true test risk. Since $Y \in \{0,1\}$, $\widehat{\mathrm{Var}}(Y) \leq \hat{\E}[Y^2] = \hat{\E}[Y]$. By the empirical Bernstein inequality applied to each fixed hypothesis in $\mathcal{H}_K$, then union-bounded over all $|\mathcal{H}_K|$ hypotheses, with probability at least $1-\delta/4$, every $(\pi_0, \Phi) \in \mathcal{H}_K$ simultaneously satisfies:
\[
    \E_N^{\pi^{\star}}[Y] \leq \hat{\E}[Y] + \sqrt{\frac{2\hat{\E}[Y]\log(5|\mathcal{H}_K|/\delta)}{n}} + \frac{3\log(5|\mathcal{H}_K|/\delta)}{n}.
\]
Since $|\mathcal{H}_K| \leq |\Pi|^{K+1}$, we have $\log(5|\mathcal{H}_K|/\delta) \leq (K+1)\log|\Pi| + \log(5/\delta) = Z$. Since the empirical step guarantees $\hat{\E}[Y] \leq \eta + 2\xi$, this implies:
\[
    \E_N^{\pi^{\star}}[Y] \leq \eta +2\xi + \sqrt{\frac{2(\eta+2\xi)Z}{n}} + \frac{3Z}{n}.
\]

\textit{Train abstention generalization.} The true stopping rate of $(\pi_0, \Phi)$ on expert trajectories is $\Pr_{M,\pi^{\star}}[\tau_{\pi_0, \Phi}(S) \leq H]$ and the empirical stopping rate over $\mathcal{D}_{\mathsf{train}}$ is $\frac{1}{m}\sum_{i=1}^m \mathbf{1}[\tau_{\pi_0,\Phi}(S_i) \leq H]$. For any fixed $(\pi_0, \Phi) \in \mathcal{H}_K$ with $\Pr_{M,\pi^{\star}}[\tau_{\pi_0, \Phi}(S) \leq H] > \frac{Z}{m}$, each of the $m$ i.i.d.\ draws $S_i \sim (M, \pi^{\star})$ independently fails to trigger abstention with probability at most $1-\frac{Z}{m}$, so the probability of observing zero empirical abstention is at most $(1-\frac{Z}{m})^m \leq e^{-Z}$. A union bound over all $|\mathcal{H}_K| \leq |\Pi|^{K+1}$ hypotheses gives failure probability at most $|\mathcal{H}_K|e^{-Z} \leq \delta/5$.

\textit{Train risk generalization.} The true train risk of $(\pi_0, \tau_{\pi_0,\Phi})$ is $\Pr_{M,\pi^{\star}}[\tau_{\pi_0, \Phi}(S) > \tau_{\pi_0, \{\pi^{\star}\}}(S)]$ and the empirical train risk is $\frac{1}{m}\sum_{i=1}^m \mathbf{1}[\tau_{\pi_0, \Phi}(S_i) > \tau_{\pi_0, \{\pi^{\star}\}}(S_i)]$. Just like above, any hypothesis with $\Pr_{M,\pi^{\star}}[\tau_{\pi_0, \Phi}(S) > \tau_{\pi_0, \{\pi^{\star}\}}(S)] > \frac{Z}{m}$ achieves zero empirical train risk with probability at most $(1-\frac{\mathcal{Z}}{m})^m \leq e^{-\mathcal{Z}}$, and a union bound over $\mathcal{H}_K$ gives failure probability at most $\delta/5$.

Taking a union bound over all failure events, with probability at least $1-\delta$, uniform convergence holds simultaneously for all three failure events across all hypotheses in $\mathcal{H}_K$. 

\paragraph{Step 4 (Bound stopping rate and regret on the learned policy's distribution).} We first address the stopping rate on $M$ under $(\pi_0, \tau_{\pi_0, \Phi})$. The generalization steps for the train abstention and train risk, respectively, give $\Pr_{M,\pi^{\star}}[\tau_{\pi_0, \Phi} \leq H] \leq \frac{Z}{m}$ and $\Pr_{M,\pi^{\star}}[\tau_{\pi_0, \Phi} > \tau_{\pi_0, \{\pi^{\star}\}}] \leq \frac{Z}{m}$. Define the safe event $B = \{\tau_{\pi_0,\Phi} \leq \tau_{\pi_0, \{\pi^{\star}\}}\}$, so $\Pr_{M,\pi^{\star}}[B^c] \leq \frac{Z}{m}$. Then,
\begin{align*}
    \Pr_{M,\pi_0}[\tau_{\pi_0, \Phi} \leq H] &\leq \Pr_{M,\pi_0}[\{\tau_{\pi_0, \Phi} \leq H\} \cap B] + \Pr_{M,\pi_0}[B^c].
\end{align*}
To apply \Cref{lemma:prefix-coupling}, let $h^{\star} = \min(\tau_{\pi_0, \Phi}, \tau_{\pi_0,\{\pi^{\star}\}})$. Both $\{\tau_{\pi_0, \Phi} \leq H\} \cap B$ and $B^c$ lie in $\mathcal{G}_{h^{\star}}$: for each $h \in [H]$, the slice $\{\tau_{\pi_0, \Phi} \leq H\} \cap B \cap \{h^{\star} = h\}$ equals $\{\tau_{\pi_0, \Phi} = h\} \cap \{\tau_{\pi_0,\{\pi^{\star}\}} \geq h\}$, which is $\mathcal{G}_h$-measurable since $\tau_{\pi_0, \Phi}$ is a stopping time and $\{\tau_{\pi_0,\{\pi^{\star}\}} \geq h\}$ depends only on whether $\pi_0$ and $\pi^{\star}$ agree on $s_1,\ldots,s_{h-1}$ (the case $h = H + 1$ is trivial as the slice is the empty set); the same argument applies to $B^c$. \Cref{lemma:prefix-coupling} therefore gives
\begin{align*}
    \Pr_{M,\pi_0}[\{\tau_{\pi_0, \Phi} \leq H\} \cap B] = \Pr_{M,\pi^{\star}}[\{\tau_{\pi_0, \Phi} \leq H\} \cap B] \leq \Pr_{M,\pi^{\star}}[\tau_{\pi_0, \Phi} \leq H] \leq \frac{Z}{m},
\end{align*}
and $\Pr_{M,\pi_0}[B^c] = \Pr_{M,\pi^{\star}}[B^c] \leq \frac{Z}{m}$. Combining yields $\Pr_{M,\pi_0}[\tau_{\pi_0, \Phi} \leq H] \leq \frac{2Z}{m}$.

To bound the stopped regret on $N$, note that $\Pr_{N,\pi_0}[B^c] = \Pr_{N,\pi^{\star}}[B^c] \leq \eta + 2\xi + \sqrt{\frac{2(\eta+2\xi)Z}{n}} + \frac{3Z}{n}$ by the test generalization guarantee and \Cref{lemma:prefix-coupling}. Decomposing the expected stopped cost of $(\pi_0, \tau_{\pi_0, \Phi})$ over $B$ and $B^c$,
\begin{align*}
 \Regret_N(\pi_0, \tau_{\pi_0, \Phi};c)
 &= \E_N^{\pi_0}\left[\Bigl(\sum_{h=1}^{\tau_{\pi_0, \Phi}-1} c_h\Bigr)\mathbf{1}(B)\right] + \E_N^{\pi_0}\left[\Bigl(\sum_{h=1}^{\tau_{\pi_0, \Phi}-1} c_h\Bigr)\mathbf{1}(B^c)\right] - \E_N^{\pi^{\star}}\left[\sum_{h=1}^{\tau_{\pi_0, \Phi} - 1} c_h\right] \\
 &\leq \E_N^{\pi^{\star}}\left[\Bigl(\sum_{h=1}^{\tau_{\pi_0, \Phi}-1} c_h\Bigr)\mathbf{1}(B)\right] + \costmax \cdot \Pr_{N,\pi_0}[B^c] - \E_N^{\pi^{\star}}\left[\sum_{h=1}^{\tau_{\pi_0, \Phi}-1} c_h\right] \\
 &= -\E_N^{\pi^{\star}}\left[\Bigl(\sum_{h=1}^{\tau_{\pi_0, \Phi}-1} c_h\Bigr)\mathbf{1}(B^c)\right] + \costmax \cdot \Pr_{N,\pi_0}[B^c] \\
 &\leq \costmax \cdot \Pr_{N,\pi_0}[B^c]\tag{costs are nonnegative}\\
 &\leq \costmax\left(\eta + 2\xi + \sqrt{\frac{2(\eta + 2\xi) Z}{n}} + \frac{3Z}{n}\right),
\end{align*}
where the first inequality applies \Cref{lemma:prefix-coupling} on $B$ (because $\tau_{\pi_0, \Phi} \leq \tau_{\pi_0, \{\pi^{\star}\}}$ on $B$, the selective policy's expected prefix cost $\sum_{h=1}^{\tau_{\pi_0, \Phi}-1} c_h$ matches the expert's expected cost over that prefix) and bounds the total cost on $B^c$ by $\costmax$. 

\subsection{Transferring the completeness guarantee to the test environment}\label{subsec: test-side abstention}

While our objective does not require bounding the test-side stopping rate $\alpha_N(\pi_0,\tau) \coloneqq \Pr_{N,\pi_0}[\tau \le H]$ distribution-free, it is natural to ask how this rate scales if the environment shift is bounded. Let $d_{\pi^\star}^{\mathrm{state}}(M,N)$ denote the total variation distance between the state-trajectory distributions induced by the expert in $M$ and $N$.

Define the safe event $B = \{\tau \le \tau_{\pi_0, \{\pi^{\star}\}}\}$ and let $h^{\star} = \min(\tau, \tau_{\pi_0, \{\pi^{\star}\}})$. Both $\{\tau \le H\} \cap B$ and $B^c$ lie in $\mathcal{G}_{h^{\star}}$ by the same measurability argument as in the proof of \Cref{theorem: deterministic rejection sample complexity}. By a union bound and two applications of \Cref{lemma:prefix-coupling},
\begin{align*}
    \alpha_N(\pi_0,\tau_{\pi_0, \Phi}) 
    &\le \Pr_{N,\pi_0}[\{\tau_{\pi_0, \Phi} \le H\} \cap B] + \Pr_{N,\pi_0}[B^c] \\
    &= \Pr_{N,\pi^\star}[\{\tau_{\pi_0, \Phi} \le H\} \cap B] + \Pr_{N,\pi^\star}[B^c] \\
    &\le \Pr_{N,\pi^\star}[\tau_{\pi_0, \Phi} \le H] + \Pr_{N,\pi^\star}[\tau_{\pi_0, \Phi} > \tau_{\pi_0, \{\pi^{\star}\}}].
\end{align*}
For a fixed validator set, the abstention event $\{\tau \le H\}$ is strictly a measurable function of the state trajectory. Thus, its probability under the expert can differ between the test environment ($N$) and training environment ($M$) by at most the TV distance:
\[
    \Pr_{N,\pi^\star}[\tau_{\pi_0, \Phi} \le H] \le \Pr_{M,\pi^\star}[\tau_{\pi_0, \Phi} \le H] + d_{\pi^\star}^{\mathrm{state}}(M,N).
\]
Finally, substituting the high-probability bounds from the proof of \Cref{theorem: deterministic rejection sample complexity} (where the train-side abstention is bounded by $Z/m$ and the late-stop risk by $\eta + \sqrt{2\eta Z/n} + 3Z/n$) yields
\[
    \alpha_N(\pi_0,\tau) 
    \le 
    \underbrace{\frac{Z}{m} + d_{\pi^\star}^{\mathrm{state}}(M,N)}_{\text{PQ-style TV transfer}} 
    + 
    \underbrace{\eta + \sqrt{\frac{2\eta Z}{n}} + \frac{3Z}{n}}_{\text{Sequential penalty}}.
\]
The first two terms are the sequential analogue of the rejection-rate transfer in batch selective classification \citep{goldwasser2020beyond}: the abstention event is a fixed measurable set, so its probability shifts by at most the TV distance. The final term has no batch analogue; it arises because the deployed policy's actions can steer the trajectory out of the safe prefix before the validators stop it.

\subsection{Proof of \Cref{prop:validator-noregret}}

By the definition of $\bar q$, the expected coverage for any fixed $\pi \in \Pi_{\mathsf{version}}$ is exactly the average payoff over all rounds:
\[
    \E_{\Phi \sim \bar q}\left[ \frac{1}{n}\sum_{j=1}^n \mathbf{1}\left[\tau_{\pi_0,\Phi}(T_j) > \tau_{\pi_0,\{\pi\} }(T_j)\right] \right] = \frac{1}{T}\sum_{t=1}^T u^t(\pi).
\]
By the definition of external regret, this average payoff is bounded by the algorithm's average expected payoff plus the regret term:
\[
    \frac{1}{T}\sum_{t=1}^T u^t(\pi) \le \frac{1}{T}\sum_{t=1}^T \E_{\pi' \sim p^t}[u^t(\pi')] + \frac{\mathrm{Reg}_T}{T}.
\]
It therefore remains to control the average expected payoff of the mixed strategy $p^t$ against the sampled validator set $\Phi^t$. 

Conditioned on the history $\mathcal{H}_{t-1}$ up to round $t-1$, the distribution $p^t$ is fixed, and $\Phi^t$ is obtained by drawing $K=\lceil 1/\rho \rceil$ policies i.i.d.\ from $p^t$. The exchangeability argument from \Cref{lemma:sparse-valid-validator-distribution} bounds the conditional expected payoff:
\[
    \E\Big[ \E_{\pi' \sim p^t}[u^t(\pi')] \;\Big|\; \mathcal{H}_{t-1} \Big] \le \frac{1}{K+1} < \rho.
\]
Since the expected payoffs $\E_{\pi' \sim p^t}[u^t(\pi')]$ are bounded in $[0,1]$, we can apply the Azuma--Hoeffding inequality. With probability at least $1-\delta$, the empirical average concentrates around its conditional expectation upper bound:
\[
    \frac{1}{T}\sum_{t=1}^T \E_{\pi' \sim p^t}[u^t(\pi')] \le \rho + \sqrt{\frac{\log(1/\delta)}{2T}}.
\]
Combining these bounds completes the proof.

\subsection{Oracle-efficient constructions of a validator distribution}\label{subsec: proof of oracle efficiency}

As mentioned in \Cref{section: implementation}, the generic no-regret construction of a sparse validator distribution is inefficient. In this section, we instantiate the generalized FTPL framework of \cite{dudik2020oracle} to obtain an oracle-efficient implementation of the no-regret strategy required by \Cref{alg:sparse-validator-dist}.

The generalized FTPL framework guarantees efficient no-regret learning, provided an online game can be reduced to a linearly perturbed offline optimization problem. To achieve this, we must map our abstract policy class into a structured vector space via a translation matrix. Crucially, this matrix must satisfy two structural conditions to ensure the random perturbations effectively smooth the leader's choices: admissibility (the matrix must have bounded entries and separate distinct policies to control the geometric complexity of the action space) and implementability (any payoff generated by the adversary must be expressible as a linear combination of the matrix's columns). If these conditions hold, FTPL can simulate the online game using synthetic datasets passed to a standard offline optimization oracle.

We now formulate the online learning problem by defining the strategy space $X$, the adversary action space $Y$, and the translation matrix $\Gamma$. This is very similar to the online learning problem constructed in \Cref{section: implementation}, except slightly more pedantic to ensure that our game satisfies the technical conditions required by generalized FTPL.
\begin{itemize}
    \item \textbf{Strategy space $X$:} Quotient $\Pi_{\mathsf{version}}$ by empirical stopping-time equivalence: $\pi \sim \pi' \iff \tau_{\pi_0,\{\pi\}}(T_j)=\tau_{\pi_0, \{\pi'\}}(T_j)$ for all $j\in[n]$. Let $X$ denote the set of equivalence classes, and for $\pi \in X$ write $x_\pi \coloneqq (\tau_{\pi_0,\{\pi\}}(T_1),\dots,\tau_{\pi_0,\{\pi\}}(T_n)) \in \{1,\dots,H+1\}^n$. 
    
    \item \textbf{Adversary action space $Y$:} Let $Y \coloneqq \{1,\dots,H+1\}^n$. For any validator set $\Phi$, we define its empirical cutoff vector $h^\Phi \coloneqq (\tau_{\pi_0,\Phi}(T_1),\dots,\tau_{\pi_0,\Phi}(T_n))$, which is an element of $Y$.
\end{itemize}

We define the payoff function $f \colon X \times Y \to [0,1]$ as $f(\pi,h)\coloneqq \frac{1}{n}\sum_{j=1}^n \mathbf{1}[x_{\pi,j}<h_j]$. Observe that for any validator set $\Phi$, the payoff evaluates to $f(\pi,h^\Phi)=\frac{1}{n}\sum_{j=1}^n \mathbf{1}[\tau_{\pi_0,\Phi}(T_j)>\tau_{\pi_0,\{\pi\}}(T_j)]$. Thus, the validator-distribution problem is exactly equivalent to online learning over $X$ against adversary actions in $Y$ with payoff function $f$.

Finally, we define the generalized FTPL translation matrix $\Gamma$. Its rows are indexed by the strategy space $X$, its columns by the environment time steps $(j,c)\in [n]\times \{1,\dots,H\}$, and its entries map the policies into structured cumulative features:
\[
    \Gamma_{\pi,(j,c)} \coloneqq \mathbf{1}[\tau_{\pi_0,\{\pi\}}(T_j) \leq c].
\]
We now prove that $\Gamma$ satisfies the technical conditions necessary for generalized FTPL.

\begin{lemma}[Admissibility and implementability of the cutoff matrix]
\label{lemma:cutoff-matrix}
The matrix $\Gamma$ is $1$-admissible and implementable with complexity $1$, in the sense of \cite{dudik2020oracle}.
\end{lemma}

\begin{proof}
We first prove admissibility. Let $\pi,\pi' \in X$ be distinct rows. Since they represent distinct stopping-time vectors, there exists some trajectory $j$ such that $\tau_{\pi_0,\{\pi\}}(T_j)\neq \tau_{\pi_0,\{\pi'\}}(T_j)$. Assuming without loss of generality that $\tau_{\pi_0,\{\pi\}}(T_j)<\tau_{\pi_0,\{\pi'\}}(T_j)$, we set $c^{\star}\coloneqq \tau_{\pi_0,\{\pi'\}}(T_j) - 1$. Then $\Gamma_{\pi,(j,c^\star)}=1$ and $\Gamma_{\pi',(j,c^\star)}=0$. Thus any two distinct rows are separated by some column. Because all matrix entries lie in $\{0,1\}$, distinct values in a column differ by exactly $1$, making $\Gamma$ $1$-admissible.

To prove implementability, we must show that each column can be simulated by a synthetic adversary action. Fix a column $(j,c)$ and define the adversary action $y^{(j,c)}\in Y$ by $y^{(j,c)}_j=c+1$ and $y^{(j,c)}_{j'}=1$ for all $j'\neq j$. Then for any $\pi\in X$:\begin{align*}
    f(\pi,y^{(j,c)}) &= \frac{1}{n}\left( \mathbf{1}\left[\tau_{\pi_0,\{\pi\}}(T_j)<c+1\right] + \sum_{j'\neq j}\mathbf{1}\left[\tau_{\pi_0,\{\pi\}}(T_{j'})<1\right] \right) \\
    &= \frac{1}{n}\,\mathbf{1}\left[\tau_{\pi_0,\{\pi\}}(T_j) \leq c\right] = \frac{1}{n}\,\Gamma_{\pi,(j,c)},
\end{align*}
since $\tau_{\pi_0,\{\pi\}}(T_{j'})\ge 1$ always. Therefore, the weighted singleton dataset $S_{j,c}\coloneqq \{(n,y^{(j,c)})\}$ implements column $(j,c)$, because for all $\pi,\pi'\in X$, $\sum_{(w,y)\in S_{j,c}} w\bigl(f(\pi,y)-f(\pi',y)\bigr) = \Gamma_{\pi,(j,c)}-\Gamma_{\pi',(j,c)}$. Thus, each column is implementable by a single synthetic adversary action.
\end{proof}

By Corollary 2.11 of \cite{dudik2020oracle}, the admissibility and implementability of $\Gamma$ guarantee the existence of an oracle-efficient, randomized no-regret maximization player. Specifically, because our translation matrix $\Gamma$ is $1$-admissible (i.e., $\delta=1$) and contains $N = nH$ columns, the generalized FTPL framework yields an expected regret bound of $O(N\sqrt{T}/\delta) = O(nH\sqrt{T})$. It remains only to identify the specific offline optimization problem required by the oracle on each round.

It remains only to identify the specific offline optimization problem required by the oracle on each round. According to the generalized FTPL algorithm \cite{dudik2020oracle}, on each round $t$, the maximization player draws independent, non-negative perturbations $\alpha_{j,c}^{\,t}$ for every column of the translation matrix. The player must then compute the perturbed best response against the empirical history of the adversary's actions—which, in our game, are the cutoff vectors $h^{\Phi^s}$ of the validator sets $\Phi^s$ sampled on previous rounds $s < t$. Formally, the oracle must solve:
\begin{equation}\label{eq:ftpl-raw-objective}
    \arg\max_{\pi\in\Pi_{\mathsf{version}}} \left\{ \sum_{s=1}^{t-1} f(\pi,h^{\Phi^s}) + \sum_{j=1}^n\sum_{c=1}^{H} \alpha_{j,c}^{\,t}\, \Gamma_{\pi,(j,c)} \right\},
\end{equation}
where the first term represents the cumulative historical payoff and the second term injects the shared linear perturbations. We now connect this oracle optimization problem to the multiple instance learning (MIL) framework of \cite{maron1997framework}.

\paragraph{Problem \eqref{eq:ftpl-raw-objective} reduces to weighted MIL over the disagreement class.} To implement the oracle efficiently, we show that the perturbed leader problem \eqref{eq:ftpl-raw-objective} can be cast as a (weighted) MIL problem. The intuition behind this reduction is straightforward: a policy stops by step $c$ on trajectory $j$ if and only if it disagrees with the base policy $\pi_0$ on \emph{at least one} state in the trajectory prefix up to step $c$. This ``at least one'' condition exactly mirrors the definition of a positive bag in MIL. By treating each trajectory prefix as a ``bag'' of states, and any disagreement with $\pi_0$ as a positive instance label, we can formally map the FTPL objective to a realizable MIL objective.

\begin{proposition}\label{prop:exact-perturbed-best-response-mil}
For the base policy $\pi_0$ and any policy $\pi \in \Pi_{\mathsf{version}}$, define the disagreement hypothesis $h_\pi(h,s) \coloneqq \mathbf{1}[\pi_h(s) \neq \pi_{0,h}(s)]$, and let $\mathcal{H}_{\pi_0} \coloneqq \{h_\pi : \pi \in \Pi_{\mathsf{version}}\}$ be the corresponding disagreement class. On round $t$, the perturbed leader problem in \eqref{eq:ftpl-raw-objective} is exactly equivalent to the following weighted realizable multiple-instance disagreement problem over $\mathcal{H}_{\pi_0}$:
\begin{equation}\label{eq:weighted-mil-objective}
\arg\max_{\pi\in\Pi_{\mathsf{version}}} \sum_{j=1}^n\sum_{c=1}^{H} w^{t-1}_{j,c}\, \mathbf{1}\left[\exists (h,s_{jh})\in B_{j,c}: h_\pi(h,s_{jh})=1\right],
\end{equation}
where the aggregated weights and prefix bags are defined as
\[ w^{t-1}_{j,c} \coloneqq \left|\left\{s<t:\tau_{\pi_0,\Phi^s}(T_j) - 1 = c\right\}\right| + n\alpha_{j,c}^{\,t}, \qquad B_{j,c} \coloneqq \{(1,s_{j1}),\dots,(c,s_{jc})\}. \]
\end{proposition}

\begin{proof}
Multiplying the objective in \eqref{eq:ftpl-raw-objective} by $n$ (which preserves the argmax) and expanding the definitions of $f$ and $\Gamma$ yields:
\[ \arg\max_{\pi\in\Pi_{\mathsf{version}}} \left\{ \sum_{s=1}^{t-1}\sum_{j=1}^n \mathbf{1}\left[\tau_{\pi_0,\{\pi\}}(T_j) \leq \tau_{\pi_0,\Phi^s}(T_j) - 1\right] + n\sum_{j=1}^n\sum_{c=1}^{H} \alpha_{j,c}^{\,t}\, \mathbf{1}\left[\tau_{\pi_0,\{\pi\}}(T_j)\leq c\right] \right\}. \]
Regrouping the first term by the cutoff value $c$ and combining it with the perturbation term gives the simplified objective $\sum_{j=1}^n\sum_{c=1}^{H} w^{t-1}_{j,c}\, \mathbf{1}[\tau_{\pi_0,\{\pi\}}(T_j)\leq c]$. We map this to the MIL framework using the disagreement hypothesis $h_\pi(h,s) = \mathbf{1}[\pi_h(s)\neq \pi_{0,h}(s)]$. By definition, the event $\tau_{\pi_0,\{\pi\}}(T_j)\leq c$ occurs if and only if there is some step $h \leq c$ where $\pi_h(s_{jh})\neq \pi_{0,h}(s_{jh})$. This is precisely the condition that the prefix bag $B_{j,c}$ is labeled positive by $h_\pi$, establishing \eqref{eq:weighted-mil-objective}. Finally, since any $\pi \in \Pi_{\mathsf{version}}$ matches the expert on the labeled training data, $h_\pi$ is strictly zero on those states, satisfying the negative singleton constraints of realizable MIL.
\end{proof}

\paragraph{Problem \eqref{eq:ftpl-raw-objective} reduces to weighted MIL over the original class.} 
Alternatively, for finite action spaces $\mathcal{A}$, we can avoid the binary disagreement wrapper and reduce directly over $\Pi_{\mathsf{version}}$ via data augmentation. By duplicating each state in a prefix for every alternative action $a \neq \pi_{0}(s)$, we create an augmented bag that a standard multi-class MIL solver evaluates as positive if $\pi$ outputs \emph{any} of those alternative actions. This inflates bag sizes by $|\mathcal{A}| - 1$.

\begin{proposition}[Exact perturbed best response as multi-class MIL over $\Pi_{\mathsf{version}}$]
\label{prop:exact-perturbed-best-response-mil-original}
Assume a finite action space $\mathcal{A}$. On round $t$, the perturbed leader problem in \eqref{eq:ftpl-raw-objective} is exactly equivalent to the following weighted realizable multi-class multiple-instance problem over $\Pi_{\mathsf{version}}$:
\begin{equation}\label{eq:weighted-mil-objective-original}
\arg\max_{\pi\in\Pi_{\mathsf{version}}} \sum_{j=1}^n\sum_{c=1}^{H} w^{t-1}_{j,c}\, \mathbf{1}\left[\exists ((h,s_{jh}), a)\in \tilde{B}_{j,c}: \pi_h(s_{jh})=a\right],
\end{equation}
with weights $w^{t-1}_{j,c}$ as in \Cref{prop:exact-perturbed-best-response-mil} and augmented prefix bags:
\[ \tilde{B}_{j,c} \coloneqq \left\{ \bigl((h,s_{jh}), a\bigr) \;\big|\; h \in \{1,\dots,c\},\; a \in \mathcal{A} \setminus \{\pi_{0,h}(s_{jh})\} \right\}. \]
\end{proposition}

\begin{proof}
The simplified FTPL objective matches \Cref{prop:exact-perturbed-best-response-mil}; we only need to verify the positive bag condition. A policy $\pi$ disagrees with $\pi_0$ on the prefix up to step $c$ if and only if $\pi_h(s_{jh}) = a$ for some step $h \leq c$ and alternative action $a \in \mathcal{A} \setminus \{\pi_{0,h}(s_{jh})\}$. This exactly matches the condition that $\pi$ predicts at least one target label $a$ in the augmented bag $\tilde{B}_{j,c}$. Finally, since all $\pi \in \Pi_{\mathsf{version}}$ perfectly match the expert on the labeled training data, they never predict alternative actions on those states, naturally satisfying the negative singleton constraints of realizable MIL.
\end{proof}

\paragraph{Agnostic oracle and computational-statistical tradeoff.} So far, we have established the existence of an oracle-efficient subroutine to compute a validator distribution. However, the specific optimization problem, (weighted) MIL, is somewhat of an exotic class of problem and it is not clear such oracles may actually be implemented or approximated in practice. It turns out, however, that we can relax the requirement of a MIL oracle for $\Pi$ to a standard \emph{agnostic learning} oracle for $\Pi$ (in fact, we can relax this to a weaker oracle known as a \emph{reliable learning} oracle \citep{kalai2021efficient}). This is simply due to the stepwise reduction to PQ-learning of \Cref{subsec: per step pq}.

Of course, the cost is a degradation in our statistical guarantees. In particular, by moving from our trajectory-wise problem to the stepwise reduction, we solve an easier computational problem but pick up a dependence on $H$ in our sample complexity. We leave the question of whether this tradeoff is fundamental, or just a by-product of our algorithms, to future work.

\section{Proofs from the Stochastic Section}
\label{sec: stochastic appendix proofs}

\subsection{Proof of \Cref{theorem: stochastic rejection sample complexity}}\label{subsec: stochastic exact bound}

Let $(\pi_0,\tau^\theta_{\pi_0, \Phi})$ be the output of \Cref{alg: stochastic mdprejectron}. We first control the abstention probability on $M$, and then the stopped trajectory Hellinger distance on $N$.

\paragraph{Step 1 (Control abstention on $M$).} 
The quantity we want to bound is the probability that the cumulative Hellinger distance exceeds $\theta$. Define $\gamma^{\star}\coloneqq \frac{\log|\Pi| + \log(8/\delta)}{m}$, and let $\gamma$ be the version space radius which is used by \Cref{alg: stochastic mdprejectron}. By assumption on $m$, we have that $\gamma^{\star}\leq \gamma$. Thus, by \Cref{lemma: approximate minimizers are hellinger close}, with probability at least $1-\delta/4$, we have $\pi^{\star}\in \Pi_{\mathsf{version}}^\gamma$ and every $\pi \in \Pi_{\mathsf{version}}^\gamma$ satisfies
\[
D_H^2(\calP_M^\pi, \calP_M^{\pi^\star}) \le \frac{\gamma}{2} + \gamma^{\star}\le 1.5\gamma.
\]
Combining this with the Hellinger triangle inequality, we have that for every $\pi \in \Pi_{\mathsf{version}}^\gamma$,
\[
D_H^2(\calP_M^\pi, \calP_M^{\pi_0})
\le 2D_H^2(\calP_M^\pi, \calP_M^{\pi^\star}) + 2D_H^2(\calP_M^{\pi_0}, \calP_M^{\pi^\star})
\le 6\gamma.
\]
Now consider any validator $\pi \in \Phi \subseteq \Pi_{\mathsf{version}}^\gamma$. Let $E$ be the event that the cumulative Hellinger distance between $\pi$ and $\pi_0$ exceeds $\theta$. To bound the probability of $E$ under $\pi_0$, we introduce the geometric mixture policy $\bar\pi \propto \sqrt{\pi \pi_0}$, using the convention in \Cref{lemma: hellinger tensorization under geometric mixture policy}. By \Cref{lemma: hellinger tensorization under geometric mixture policy}, the trajectory affinity satisfies
\[
1 - D_H^2(\calP^\pi_M, \calP^{\pi_0}_M) \le \E_{T \sim \calP^{\bar \pi}_M}\left[e^{-\sum_{h=1}^H d_H^2(\pi_h, \pi_{0,h})}\right] \le \Pr_{\calP^{\bar\pi}_M}[E]e^{-\theta} + 1 - \Pr_{\calP^{\bar\pi}_M}[E].
\]
Rearranging this inequality yields $\Pr_{\calP^{\bar\pi}_M}[E] \le \frac{1}{1-e^{-\theta}} D_H^2(\calP^\pi_M, \calP^{\pi_0}_M) \le \frac{1+\theta}{\theta} D_H^2(\calP^\pi_M, \calP^{\pi_0}_M)$. To shift this probability from the mixture policy to the deployed policy $\pi_0$, we apply a Hellinger change-of-measure inequality (e.g., Lemma~3.1 of \citet{foster2024behavior}) alongside \Cref{lemma: geometric mixture hellinger bound}:
\begin{align*}
    \Pr_{\calP^{\pi_0}_M}[E]
    &\le 2\Pr_{\calP^{\bar\pi}_M}[E] + 4D_H^2(\calP^{\bar\pi}_M, \calP^{\pi_0}_M) \\
    &\le \frac{2(1+\theta)}{\theta}D_H^2(\calP^\pi_M, \calP^{\pi_0}_M) + 6D_H^2(\calP^\pi_M, \calP^{\pi_0}_M) \\
    &= \left(\frac{2}{\theta} + 8\right) D_H^2(\calP^\pi_M, \calP^{\pi_0}_M).
\end{align*}
Substituting our version space bound $D_H^2(\calP_M^\pi, \calP_M^{\pi_0}) \le 6\gamma$, the probability of $E$ under $\pi_0$ is at most $\left(\frac{12}{\theta} + 48\right)\gamma$. Note that the algorithm abstains ($\tau^\theta_{\pi_0, \Phi} \leq H$) if and only if at least one of these cumulative distances exceeds $\theta$. By a union bound over the $|\Phi| \le K_{\mathsf{ens}}$ validators,
\[
\alpha_M(\pi_0,g) \le K_{\mathsf{ens}}\left(\frac{12}{\theta}+48\right)\gamma.
\]

\paragraph{Step 2 (Bound Hellinger Distance of stopped trajectories on $N$).}

Let $M = H+1$. For each policy $\pi \in \Pi_{\mathsf{version}}^{\gamma}$, let $v^\pi \in [M]^n$ be its empirical stopping-time vector on the $n$ test trajectories, where the $j$-th coordinate is given by $(v^\pi)_j = \tau_{\pi_0, \{\pi\}}^{\theta}(T_j)$. Let $\mathcal{T} = \{v^\pi : \pi \in \Pi_{\mathsf{version}}^{\gamma}\}$ be the finite set of all such vectors.

By \Cref{lemma: generalized sparse selective policies} applied to $\mathcal{T}$ with target error $\eta/2$, there exists a distribution $q^{\star}$ over subsets $\Phi^{\prime} \subseteq \Pi_{\mathsf{version}}^{\gamma}$ satisfying $|\Phi^{\prime}| \le \lceil 2/\eta \rceil$ such that, because $\pi^{*} \in \Pi_{\mathsf{version}}^{\gamma}$:
\[
\E_{\Phi' \sim q^\star}\left[\frac{1}{n}\sum_{j=1}^n \mathbf{1}\left[\tau_{\pi_0, \Phi'}^\theta(T_j) > \tau_{\pi_0, \{\pi^\star\}}^\theta(T_j)\right]\right] \le \frac{\eta}{2}.
\]
For a validator set $\Phi$, define the empirical risk $\widehat p(\Phi) \coloneqq \frac{1}{n}\sum_{j=1}^{n} \mathbf{1}[\tau_{\pi_0, \{\pi^\star\}}^\theta(T_j) < \tau_{\pi_0, \Phi}^\theta(T_j)]$. By an identical argument to \Cref{theorem: deterministic rejection sample complexity} (Markov's inequality and multiple draws) we have $\widehat p(\Phi) \le \eta$ with probability at least $1 - \delta / 4$.

To generalize to expert trajectories in $N$, note the final validator set satisfies $|\Phi| \le K_{\mathsf{ens}}$. Consider the indicator class\begin{align*}
    \mathcal{Y}_{K_{\mathsf{ens}}} \coloneqq \{ T \mapsto \mathbf{1}[\tau_{\pi_0, \{\pi^\star\}}^\theta(T) < \tau_{\pi_0, \Phi}^\theta(T)] : \pi_0 \in \Pi, \Phi \subseteq \Pi, |\Phi| \le K_{\mathsf{ens}} \},
\end{align*}
which has cardinality at most $|\Pi|^{K_{\mathsf{ens}}+1}$. Applying the empirical Bernstein inequality with a union bound, with probability at least $1-\delta/4$, simultaneously for all such $(\pi_0,\Phi)$,
\[
p(\Phi) \le \widehat p(\Phi) + \sqrt{\frac{2\widehat p(\Phi) Z}{n}} + \frac{3Z}{n} \le \eta + \sqrt{\frac{2\eta Z}{n}} + \frac{3Z}{n},
\]
where $p(\Phi) \coloneqq \Pr_{T \sim (\pi^\star,N)}\left(\tau_{\pi_0, \{\pi^\star\}}^\theta(T) < \tau_{\pi_0, \Phi}^\theta(T)\right)$. Under the assumption $n \ge 8Z/\eta$, each correction term is at most $\eta/2$, yielding $p(\Phi) \le 2\eta$.

Finally, to transfer this to the deployed trajectories, let $\mathcal{E} \coloneqq \{\tau_{\pi_0, \{\pi^\star\}}^\theta(T) < \tau_{\pi_0, \Phi}^\theta(T)\}$ and define $\tau'(T) \coloneqq \tau_{\pi_0, \Phi\cup\{\pi^\star\}}^\theta(T)$. Applying \Cref{lemma: Hellinger bound in N} with validator set $\Phi \cup \{\pi^\star\}$ and comparison policy $\pi^\star$, we have deterministically $D_H^2(P^{\pi_0}_{|\tau'}, P^{\pi^\star}_{|\tau'}) \le \theta$, because the event $\tau_{\pi_0, \{\pi^\star\} }^\theta(T) < \tau_{\pi_0, \Phi\cup\{\pi^\star\}}^\theta(T)$ has probability zero under $(\pi_0,N)$. Since $\mathcal{E}$ is measurable with respect to the trajectory stopped at $\tau'$, the binary Hellinger change-of-measure inequality yields
\begin{align*}
\Pr_{(\pi_0,N)}(\mathcal{E})
&\le \left(\sqrt{\Pr_{(\pi^\star,N)}(\mathcal{E})} + \sqrt{2D_H^2\left(\calP^{\pi_{0{|\tau'}}}_N, \calP^{\pi^\star_{|\tau'}}_N\right)}\right)^2 \\
&\le 2\Pr_{(\pi^\star,N)}(\mathcal{E}) + 4D_H^2\left(\calP^{\pi_{0{|\tau'}}}_N, \calP^{\pi^\star_{|\tau'}}_N\right) \\
&\le 4\eta + 4\theta,
\end{align*}
where the last step uses $p(\Phi) \le 2\eta$. A final application of \Cref{lemma: Hellinger bound in N}, now with validator set $\Phi$, comparison policy $\pi^\star$, and failure probabilities $4\eta+4\theta$ and $2\eta$, gives\begin{align*}
    D_H^2(P^{\pi_{0{|\tau}}}_N, P^{\pi^\star_{|\tau}}_N) &\le \theta + \frac{(4\eta+4\theta)+2\eta}{2} \\
    &= 3\theta + 3\eta.
\end{align*}

The four failure events above each occur with probability at most $\delta/4$, so the claimed bounds hold simultaneously with probability at least $1-\delta$.

\subsection{Supporting Hellinger lemmas}

We collect the proofs of the four Hellinger lemmas used in the proof of \Cref{theorem: stochastic rejection sample complexity}.

\begin{lemma}[Approximate minimizers are Hellinger-close]\label{lemma: approximate minimizers are hellinger close}
Let $\pi_0$ be the MLE on a dataset of $m$ trajectories, and define $\gamma^{\star}\coloneqq (\log|\Pi| + \log(2/\delta))/m$. With probability at least $1-\delta$:
\begin{enumerate}
\item $\pi^\star\in \Pi_{\mathsf{version}}^{\gamma^\star}$, i.e., $\mathsf{LogLoss}(\pi^\star) \le \mathsf{LogLoss}(\pi_0) + \gamma^\star$.
\item Every $\pi \in \Pi$ with $\mathsf{LogLoss}(\pi) \le \mathsf{LogLoss}(\pi_0) + \gamma$ satisfies $D_H^2(\calP_M^\pi, \calP_M^{\pi^\star}) \le \gamma/2 + \gamma^\star$.
\end{enumerate}
\end{lemma}
\begin{proof}
Let $L(\pi) = \prod_{k=1}^{m}\prod_{h=1}^{H}\frac{\pi_h(a_h^{(k)}\mid s_h^{(k)})}{\pi^\star_h(a_h^{(k)}\mid s_h^{(k)})}$ denote the likelihood ratio of the dataset under $\pi$ versus $\pi^\star$. For a single trajectory, the likelihood ratio under $\calP_M^{\pi^\star}$ sums only the $\calP_M^\pi$-mass on trajectories in the support of $\calP_M^{\pi^\star}$, so its expectation is at most~1 (and exactly~1 whenever $\calP_M^\pi \ll \calP_M^{\pi^\star}$). Since $L(\pi)$ is a product of $m$ independent trajectory likelihood ratios, we have $\E[L(\pi)] \le 1$ for all $\pi$. Thus, by Markov's inequality and a union bound,
\[
    \Pr\left[\sup_{\pi\in\Pi} L(\pi) > \frac{2|\Pi|}{\delta}\right] \leq \frac{\delta}{2}.
\]
Taking the log and dividing by $m$, with probability at least $1-\delta/2$ the log-loss gap satisfies $\mathsf{LogLoss}(\pi^\star) \leq \mathsf{LogLoss}(\pi_0) + \gamma^\star$, proving Part~1.

For Part~2, if $\pi$ is a $\gamma$-approximate minimizer then $\mathsf{LogLoss}(\pi) \leq \mathsf{LogLoss}(\pi^\star) + \gamma$ (since $\pi_0$ is the MLE), which rearranges to
\begin{align}
    -\frac{1}{m}\sum_{k=1}^{m}\log\left(\prod_{h=1}^{H}\frac{\pi_h(a_h^{(k)}\mid s_h^{(k)})}{\pi^\star_h(a_h^{(k)}\mid s_h^{(k)})}\right) \leq \gamma 
    &\implies \sum_{k=1}^{m}\log\sqrt{\prod_{h=1}^{H}\frac{\pi_h(a_h^{(k)}\mid s_h^{(k)})}{\pi^\star_h(a_h^{(k)}\mid s_h^{(k)})}} \geq -\frac{m\gamma}{2} \\
    &\implies \prod_{k=1}^{m}\sqrt{\prod_{h=1}^{H}\frac{\pi_h(a_h^{(k)}\mid s_h^{(k)})}{\pi^\star_h(a_h^{(k)}\mid s_h^{(k)})}} \geq e^{-m\gamma/2}.\label{eq: log loss error rewritten}
\end{align} 
By definition, the expectation of each factor in the outer product equals $1 - D_H^2(\calP_M^\pi, \calP_M^{\pi^\star})$, so the LHS has expectation $(1-D_H^2(\calP_M^\pi, \calP_M^{\pi^\star}))^m \le e^{-m D_H^2(\calP_M^\pi,\calP_M^{\pi^\star})}$. For any $\pi$ with $D_H^2(\calP_M^\pi,\calP_M^{\pi^\star}) > \gamma/2 + \gamma^\star$, Markov's inequality applied to \eqref{eq: log loss error rewritten} gives failure probability at most $\delta/(2|\Pi|)$. A union bound over $\Pi$ completes the proof with overall probability at least $1-\delta$.
\end{proof}

\begin{lemma}[Hellinger tensorization under geometric mixture]\label{lemma: hellinger tensorization under geometric mixture policy}
    Let $\calP^{\pi}$ and $\calP^{\pi'}$ denote the trajectory distributions under policies $\pi$ and $\pi'$ in an MDP $M$. Define the geometric mixture policy by $\bar\pi_h(a\mid s) \coloneqq \sqrt{\pi_h(a\mid s)\,\pi'_h(a\mid s)}/Z_h(s)$, where $Z_h(s) \coloneqq \sum_{a} \sqrt{\pi_h(a \mid s)\,\pi'_h(a \mid s)}$.\footnote{If $Z_h(s)=0$, define $\bar\pi_h(\cdot\mid s)$ arbitrarily. In this case $\pi_h(\cdot\mid s)$ and $\pi'_h(\cdot\mid s)$ have disjoint support, so $\sqrt{\pi_h(a\mid s)\pi'_h(a\mid s)}=0$ for every $a$, and the identities below still hold.} Then
    \[
        1 - D_H^2(\calP^{\pi}, \calP^{\pi'}) = \E_{T \sim \calP^{\bar \pi}}\left[\prod_{h=1}^H \bigl(1 - d_H^2(\pi_h(\cdot \mid s_h), \pi'_h(\cdot \mid s_h))\bigr)\right].
    \]
\end{lemma}
\begin{proof}
By definition, the trajectory-level squared Hellinger affinity is
\[
    1 - D_H^2(\calP^{\pi}, \calP^{\pi'}) = \sum_T \sqrt{\calP^{\pi}(T)\,\calP^{\pi'}(T)}.
\]
Expanding each trajectory probability,
\[
    \sqrt{\calP^{\pi}(T)\,\calP^{\pi'}(T)}
    =
    P_0(s_1)
    \left(\prod_{h=1}^{H-1} P_h(s_{h+1} \mid s_h, a_h)\right)
    \left(\prod_{h=1}^{H} \sqrt{\pi_h(a_h \mid s_h)\,\pi'_h(a_h \mid s_h)}\right).
\]
Let $Z_h(s_h) \coloneqq \sum_{a} \sqrt{\pi_h(a \mid s_h)\,\pi'_h(a \mid s_h)} = 1 - d_H^2(\pi_h(\cdot \mid s_h), \pi'_h(\cdot \mid s_h))$. With the convention in the lemma statement, the identity $\sqrt{\pi_h(a_h \mid s_h)\,\pi'_h(a_h \mid s_h)} = \bar\pi_h(a_h \mid s_h)\,Z_h(s_h)$ holds for every $s_h,a_h$: it is the definition when $Z_h(s_h)>0$, and both sides are zero when $Z_h(s_h)=0$. Substituting,
\[
    \sqrt{\calP^{\pi}(T)\,\calP^{\pi'}(T)} = \calP^{\bar\pi}(T)\prod_{h=1}^H Z_h(s_h).
\]
Summing over all trajectories $T$ gives the result.
\end{proof}

\begin{lemma}[Geometric mixture Hellinger bound]\label{lemma: geometric mixture hellinger bound}
    Let $\pi, \pi', \bar\pi$ be as in \Cref{lemma: hellinger tensorization under geometric mixture policy} and let $\gamma = D_H^2(\calP^{\pi}, \calP^{\pi'}).$ Then
    \[
        D_H^2(\calP^{\bar \pi}, \calP^{\pi'}) \leq \frac{3}{2}D_H^2(\calP^{\pi}, \calP^{\pi'}).
    \]
\end{lemma}
\begin{proof}
Define
\[
P := \mathcal P^\pi,
\qquad
Q := \mathcal P^{\pi'},
\qquad
\bar P := \mathcal P^{\bar\pi}.
\]

From the proof of \Cref{lemma: hellinger tensorization under geometric mixture policy}, we have that pointwise $\bar{P}(T) Z(T) = \sqrt{P(T)Q(T)}$ where $Z(T) = \prod_{h=1}^H Z_h(s_h) \le 1$. Therefore $\bar{P}(T) \geq \sqrt{P(T)Q(T)}$, and
\begin{align*}
    1 - D_H^2(\bar{P}, Q) &= \sum_T \sqrt{\bar{P}(T)Q(T)} \geq \sum_T \sqrt{\sqrt{P(T)Q(T)} \cdot Q(T)} 
    = \sum_T P(T)^{1/4}Q(T)^{3/4}.
\end{align*}
Applying Hölder's inequality with conjugate exponents $3/2$ and $3$, we obtain
\begin{align*}
\sum_T P(T)^{1/2}Q(T)^{1/2}
&=
\sum_T
\bigl(P(T)^{1/4}Q(T)^{3/4}\bigr)^{2/3} P(T)^{1/3} \\
&\le
\left(
\sum_T
\Bigl[\bigl(P(T)^{1/4}Q(T)^{3/4}\bigr)^{2/3}\Bigr]^{3/2}
\right)^{2/3}
\left(
\sum_T \bigl(P(T)^{1/3}\bigr)^3
\right)^{1/3} \\
&=
\left(
\sum_T P(T)^{1/4}Q(T)^{3/4}
\right)^{2/3}
\left(
\sum_T P(T)
\right)^{1/3} \\
&=
\left(
\sum_T P(T)^{1/4}Q(T)^{3/4}
\right)^{2/3}.
\end{align*}
Hence $1 - D_H^2(\calP^{\pi}, \calP^{\pi'}) \leq \left(\sum_T P(T)^{1/4}Q(T)^{3/4}\right)^{2/3}$. Raising to the $3/2$ power, $\sum_T P(T)^{1/4}Q(T)^{3/4} \geq (1-D_H^2(\calP^{\pi}, \calP^{\pi'}))^{3/2}$, so by Bernoulli's inequality: \[D_H^2(\bar{P}, Q) \leq 1 - (1-D_H^2(\calP^{\pi}, \calP^{\pi'}))^{3/2} \leq \frac32D_H^2(\calP^{\pi}, \calP^{\pi'}).\] 
\end{proof}

\begin{lemma}[Stopped trajectory Hellinger bound]\label{lemma: Hellinger bound in N}  
    Let $\pi_0, \pi' \in \Pi$ and $\Phi \subseteq \Pi$. Let $\tau \coloneqq \tau_{\pi_0, \Phi}^\theta(T)$. If $\Phi$ satisfies
    \begin{align*}
        \Pr_{T \sim (\pi_0, N)}\bigl(\tau_{\pi_0, \{\pi'\}}^\theta(T) < \tau\bigr) &< \eta \quad\text{and}\quad\Pr_{T \sim (\pi', N)}\bigl(\tau_{\pi_0, \{\pi'\} }^\theta(T) < \tau\bigr) < \eta'
    \end{align*}
    then, with probability at least $1 - \eta$ over $T \sim (\pi_0, N)$, the cumulative per-step Hellinger distance up to $\tau$ is bounded by $\theta$:
    \[
        \sum_{h=1}^{\tau - 1} d_H^2(\pi'_h(\cdot \mid s_h), \pi_{0,h}(\cdot \mid s_h)) \leq \theta.
    \]
    Furthermore, the stopped trajectory distributions satisfy
    \[
        D_H^2\bigl( \calP^{\pi_{0|\tau}}_N, \calP^{\pi'_{|\tau}}_N \bigr) \leq \theta + \frac{\eta + \eta'}{2}.
    \]
\end{lemma}
\begin{proof}
Let
\[
\tau' \coloneqq \tau_{\pi_0,\{\pi'\}}^\theta(T),
\qquad
\mathcal E \coloneqq \{\tau \le \tau'\}.
\] By hypothesis, $\Pr_{(\pi_0,N)}(\mathcal{E}^c) < \eta$, so $\Pr(\mathcal{E}) \geq 1 - \eta$. Conditioned on $\mathcal{E}$, the sum \[\sum_{h=1}^{\tau-1} d_H^2(\pi'_h(\cdot|s_h), \pi_{0,h}(\cdot|s_h)) \le \theta\] by the definition of $\tau'$ and monotonicity of the sum.

For the second claim, define $\tau^\star(T) \coloneqq \min(\tau, \tau')$. Since $\tau^{\star}\leq \tau'$ on every trajectory, we have $\sum_{h < \tau^\star} d_H^2(\pi'_h(\cdot \mid s_h), \pi_{0,h}(\cdot \mid s_h)) \leq \theta$ pathwise. The stopped affinity therefore satisfies\begin{align*}
    \sum_T \sqrt{\calP^{\pi_{0|\tau^\star}}_N(T) \calP^{\pi'_{|\tau^\star}}_N(T)}
    &= \E_{T \sim \calP^{\bar \pi}}\left[\prod_{h < \tau^\star}\bigl(1 - d_H^2(\pi'_h(\cdot|s_h),\pi_{0,h}(\cdot|s_h))\bigr)\right]\\
    &\geq \E_{T \sim \calP^{\bar \pi}}\left[1 - \sum_{h<\tau^\star} d_H^2(\pi'_h(\cdot|s_h),\pi_{0,h}(\cdot|s_h))\right]\geq 1 - \theta,
\end{align*}
so $D_H^2(\calP^{\pi_{0|\tau^\star}}_N, \calP^{\pi'_{|\tau^\star}}_N) \leq \theta$.

Write $A(P,Q) = \sum_T \sqrt{P(T)Q(T)}$ for the Hellinger affinity. On $\mathcal{E}$, the stopped distributions at $\tau$ and $\tau^\star$ agree, so
\[
    A(\calP^{\pi_{0|\tau}}_N, \calP^{\pi'_{|\tau}}_N) \ge \sum_{T \in \mathcal{E}} \sqrt{\calP^{\pi_{0|\tau^\star}}_N(T) \calP^{\pi'_{|\tau^\star}}_N(T)}.
\]
On the other hand,
\[
    A(\calP^{\pi_{0|\tau^\star}}_N, \calP^{\pi'_{|\tau^\star}}_N) = \sum_{T \in \mathcal{E}} \sqrt{\calP^{\pi_{0|\tau^\star}}_N(T) \calP^{\pi'_{|\tau^\star}}_N(T)} + \sum_{T \in \mathcal{E}^c} \sqrt{\calP^{\pi_{0|\tau^\star}}_N(T) \calP^{\pi'_{|\tau^\star}}_N(T)}.
\]
By AM-GM, the second sum satisfies
\begin{align*}
\sum_{T \in \mathcal{E}^c} \sqrt{\calP^{\pi_{0|\tau^\star}}_N(T) \calP^{\pi'_{|\tau^\star}}_N(T)}&\le \sum_{T\in \mathcal E^c} \frac{\calP^{\pi_{0|\tau^\star}}_N(T) + \calP^{\pi'_{|\tau^\star}}_N(T)}{2}\le \frac{\eta+\eta'}{2}
\end{align*} Therefore $    A(\calP^{\pi_{0|\tau}}_N, \calP^{\pi'_{|\tau}}_N) \ge A(\calP^{\pi_{0|\tau^\star}}_N, \calP^{\pi'_{|\tau^\star}}_N) - \frac{\eta+\eta'}{2},$
which gives 
\begin{align*}
    D_H^2(\calP^{\pi_{0|\tau}}_N, \calP^{\pi'_{|\tau}}_N) &\le D_H^2(\calP^{\pi_{0|\tau^\star}}_N, \calP^{\pi'_{|\tau^\star}}_N) + (\eta+\eta')/2 \le \theta + \frac{\eta+\eta'}2
\end{align*}
\end{proof}

\subsection{Proof of \Cref{thm: mixed regret bound}}

Because $\pi_{\mathsf{sw}}$ and $\pi^\star$ are valid policies evaluated over the full-horizon MDP $N$, we can directly invoke Theorem 3.1 of \cite{foster2024behavior} to bound their regret difference:
\[
    J(\pi_{\mathsf{sw}}) - J(\pi^\star) \le \sqrt{6\sigma_{\pi^\star}^2 \cdot D_H^2(\calP^{\pi_{\mathsf{sw}}}, \calP^{\pi^\star})} + O\left(\costmax \log(\costmax\epsilon^{-1})\right) \cdot D_H^2(\calP^{\pi_{\mathsf{sw}}}, \calP^{\pi^\star}) + \epsilon,
\]
where $\calP^{\pi_{\mathsf{sw}}}$ and $\calP^{\pi^\star}$ are the full-horizon trajectory distributions of $\pi_{\mathsf{sw}}$ and $\pi^{\star}$, respectively. 

To evaluate the full-horizon Hellinger distance, we expand the Hellinger affinity between $\calP^{\pi_{\mathsf{sw}}}$ and $\calP^{\pi^\star}$. Recall that the squared Hellinger distance satisfies $1 - D_H^2(P, Q) = \sum_T \sqrt{P(T)Q(T)}$. Let $T_{|\tau}$ denote the prefix of the trajectory up to the stopping time $\tau$, and let $T_{>\tau}$ denote the suffix. Therefore, the full-horizon trajectory distribution for any policy $\pi$ factors as $P^\pi(T) = P^\pi(T_{|\tau}) P^\pi(T_{>\tau} \mid T_{|\tau})$.

By construction, the mixed policy $\pi_{\mathsf{sw}}$ executes the learner $\hat{\pi}$ up to $\tau$ and the expert $\pi^\star$ for all steps $h > \tau$. Thus, its prefix distribution is exactly the stopped learner distribution, $\calP^{\pi_{\mathsf{sw}}}(T_{|\tau}) = \calP^{\hat{\pi}_{|\tau}}(T_{|\tau})$. Crucially, because $\tau$ is a stopping time, the event $\{\tau = t\}$ depends exclusively on the trajectory prefix $T_{1:t}$. Consequently, for any prefix where the switch occurs at step $t$, the conditional distribution of the remaining trajectory is governed entirely by the expert policy, yielding $\calP^{\pi_{\mathsf{sw}}}(T_{t+1:H} \mid T_{1:t}) = \calP^{\pi^\star}(T_{t+1:H} \mid T_{1:t})$. Thus,
\begin{align*}
    1 - D_H^2(\calP^{\pi_{\mathsf{sw}}}, \calP^{\pi^\star}) &= \sum_{T_{|\tau}} \sum_{T_{>\tau}} \sqrt{\calP^{\pi_{\mathsf{sw}}}(T_{|\tau}) \calP^{\pi_{\mathsf{sw}}}(T_{>\tau} \mid T_{|\tau}) \cdot \calP^{\pi^\star}_{|\tau}(T_{|\tau}) \calP^{\pi^\star}(T_{>\tau} \mid T_{|\tau})} \\
    &= \sum_{T_{|\tau}} \sqrt{\calP^{\hat{\pi}_{|\tau}}(T_{|\tau}) \calP^{\pi^\star_{|\tau}}(T_{|\tau})} \sum_{T_{>\tau}} \calP^{\pi^\star}(T_{>\tau} \mid T_{|\tau}).
\end{align*}
Because $P^{\pi^\star}(T_{>\tau} \mid T_{|\tau})$ is a valid probability distribution over the suffix trajectories, the inner sum evaluates exactly to $1$ for every prefix $T_{|\tau}$. The remaining sum is exactly the Hellinger affinity of the stopped trajectories. Therefore, the full-horizon trajectory distance perfectly collapses to the distance accumulated up to the stopping time:\begin{align}
    D_H^2(\calP^{\pi_{\mathsf{sw}}}, \calP^{\pi^\star}) = D_H^2\bigl(\calP^{\hat{\pi}_{|\tau}},\, \calP^{\pi^\star_{|\tau}}\bigr).\label{eq: stopped hellinger equality}
\end{align}
Substituting \eqref{eq: stopped hellinger equality} into the bound from Theorem 3.1 of \cite{foster2024behavior} concludes the proof.

\subsection{Proof of \Cref{theorem: stochastic lower bound}}

Let the state space be $\mathcal{X} = \{1, 2, \dots, N\}$. Let $P$ be uniform over $\mathcal{X}$, $Q$ be uniform over a known subset $S_Q = \{1, \dots, d\}$, and let $\Delta \in [0,1/2]$ be a parameter to be chosen later. Let $\Pi$ be the class of stochastic policies parameterized by $\sigma \in \{-1, 1\}^d$: for $x \in S_Q$, the expert $\pi^\star$ takes action $1$ with probability $1/2 + \sigma_x \Delta$; for $x \notin S_Q$, the probability is $1/2$. Similarly, define the cost function $c_\sigma : \mathcal X \times \{0,1\} \rightarrow [0,1]$ by $c_\sigma (x, 1) = 1\{\sigma_x = -1\}$ and $c_\sigma (x,0)=1\{\sigma_x = 1\}$ for $x\in S_Q$ and 0 otherwise. 

We now apply Yao's minimax principle. It is enough to lower bound the expected performance of a deterministic proper learner under the uniform prior $\sigma \sim \text{Unif}\left(\{-1,1\}^d\right)$. Let $D$ be the $m$ labeled samples from $P$ the learner observes. Since the learner is deterministic, after observing $D$, it will, for each state $x$, choose an acceptance probability $\alpha_x(D) \in [0,1]$ and a bernoulli parameter $q_x(D) = \hat \pi (1|x)$. Crucially, because the learner is proper, it must output a policy in $\Pi$, meaning $q_x(D) \in \{1/2 - \Delta, 1/2 + \Delta\}$.

For $x\in S_Q$, define the one-step selective regret at state $x$ as\begin{align*}
    R_x(D,\sigma)
&\coloneqq
\mathbb E_{a\sim \hat\pi_D(\cdot\mid x)}\!\left[c_\sigma(x,a)\right]
-
\mathbb E_{a\sim \pi_\sigma(\cdot\mid x)}\!\left[c_\sigma(x,a)\right]\\
&=\Delta+\sigma_x\!\left(\frac12-q_x(D)\right).\tag{direct computation}
\end{align*}
Now, define $\eta_x = \mathbb E[\sigma_x |D]$. By conditioning on $D$, we get 
\[
\mathbb E[R_x(D,\sigma)|D] = \Delta + \eta_x\!\left(\frac12-q_x(D)\right) \ge \Delta -\Delta |\eta_x|
\] 
where we used the proper learner constraint $\left|\frac12 - q_x(D)\right| = \Delta$. 

Multiplying by $\alpha_x$ and then taking the expectation over $D$ gives 
\[\mathbb E[\alpha_x(D) R_x(D,\sigma)] \ge \Delta \mathbb E[\alpha_x(D)] - \Delta \mathbb E[|\eta_x|]\]
where the full expectation is now over $\sigma$ and $D$. 

Let $n_x(D) \sim \mathrm{Binomial}(m, 1/N)$ be the number of times state $x$ is observed. The posterior bias satisfies 
\begin{align*}
    \mathbb E[|\eta_x(D)| \mid n_x(D)=k] &= \TV\left(\text{Binomial}\left(k, \frac12 + \Delta\right),\text{Binomial}\left(k, \frac12 - \Delta\right)\right)\\
    &\le 2\sqrt2 \Delta \sqrt k
\end{align*}
by Pinsker's inequality. Taking expectations yields $\mathbb E[|\eta_x(D)|] \le 2\sqrt2 \Delta \mathbb E\left[\sqrt{n_x (D)}\right]$.
Then, averaging over the possibilities of $x$ gives 
\[\mathbb E[\Regret_Q(\hat\pi,\tau;c)] \ge \Delta \mathbb E[\overline \alpha(D)] - \frac{2\sqrt2\Delta^2}{d} \sum_{x=1}^d \mathbb E\left[\sqrt{n_x(D)}\right]\]
where $\overline \alpha(D) = \frac1d \sum_{x=1}^d \alpha_x(D)$. Jensen's inequality on the second expectation then yields 
\[\mathbb E[\Regret_Q(\hat\pi,\tau;c)] \ge \Delta \mathbb E[\overline \alpha(D)] - 2{\sqrt2\Delta^2}\sqrt{\frac mN}.\]

We now bound the other expectation. Since $P$ is uniform, $\mathbb E_{x\sim P}[1-\alpha_x(D)] = 1- \frac1N \sum_{x=1}^N \alpha_x(D)$. Thus, if the learner satisfies $\mathbb E\!\left[\mathbb E_{x\sim P}[1-\alpha_x(D)]\right]\le \epsilon$, then $\frac1N\sum_{x=1}^N \mathbb E[\alpha_x(D)]\ge 1-\epsilon$. Even if the learner accepts with probability \(1\) on all states outside \(S_Q\), this still forces $\mathbb E[\bar\alpha(D)]\ge 1-\frac{N\epsilon}{d}$.

Now, setting $N= \frac{d}{2\epsilon}$ and $\Delta = 4\epsilon$ gives us the bounds 
\[E[\Regret_Q(\hat\pi,\tau;c)] \ge 2\epsilon - 32\sqrt{2}\epsilon^{2} \sqrt{\frac {2\epsilon m}d}.\]
Requiring this expected regret to be at most $\epsilon$ yields:
\[
1 \le 32\sqrt{2}\epsilon \sqrt{\frac{2\epsilon m}{d}} \implies m \ge \frac{d}{4096\epsilon^3}.
\]
Setting $c_0 = 1/4096$ completes the proof.

\section{Proofs from the Misspecification Section}
\paragraph{Notation.} For a base policy $\pi_0 \in \Pi$, comparator $\pi \in \Pi$, and validator sequence $\Phi$, define the expert-side late-stop risk
\[
    r_N(\pi_0,\pi,\Phi)\coloneqq\Pr_{T\sim P_N^\star}\left[\tau_{\pi_0,\Phi}(T)>\tau_{\pi_0,\{\pi\}}(T)\right],
\]
with empirical analogue
\[
    \widehat r_N(\pi_0,\pi,\Phi)\coloneqq\frac{1}{n}\sum_{j=1}^n \mathbf 1\left[\tau_{\pi_0,\Phi}(T_j)>\tau_{\pi_0,\{\pi\}}(T_j)\right],
\]
where $T_1,\dots,T_n$ are the unlabeled expert test trajectories. This is the expert-side late-stop quantity that appears in the equilibrium lemma and in the generalization argument.

\subsection{Proof of \Cref{lem:symmetrically-regularized-equilibrium}}

Define a zero-sum game in which the maximizer chooses $p \in \Delta(\Pi)$ and the minimizer chooses $q \in \Delta(\Pi^K)$, with payoff
\begin{align*}
    U(p, q) \coloneqq \mathbb{E}_{\pi \sim p}\, \mathbb{E}_{\Phi \sim q}\left[
    \widehat{r}_N(\pi_0, \pi, \Phi) - \Lambda \cdot \widehat{d}_M(\pi) + \frac{\Lambda}{K} \sum_{k=1}^K \widehat{d}_M(\phi_k)
    \right].
\end{align*}
Since $U$ is bilinear, von Neumann's minimax theorem implies $\min_{q}\max_{p} U(p,q) = \max_{p}\min_{q} U(p,q)$. To bound the maximin, fix any $p \in \Delta(\Pi)$ and let $p^K$ denote the strategy of forming $\Phi = (\phi_1, \dots, \phi_K)$ by drawing $\phi_1, \dots, \phi_K$ i.i.d.\ from $p$. Then,
\begin{align*}
    \min_{q}\; U(p, q) \leq U(p, p^K) &= \mathbb{E}_{\pi \sim p,\, \Phi \sim p^K}\big[\widehat{r}_N(\pi_0, \pi, \Phi)\big] - \Lambda \cdot \mathbb{E}_{\pi \sim p}\big[\widehat{d}_M(\pi)\big] + \frac{\Lambda}{K} \sum_{k=1}^K \mathbb{E}_{\phi_k \sim p}\big[\widehat{d}_M(\phi_k)\big]\\
    &= \mathbb{E}_{\pi \sim p,\, \Phi \sim p^K}\big[\widehat{r}_N(\pi_0, \pi, \Phi)\big] \tag{the other two terms cancel}\\
    &\leq \frac{1}{K + 1}.\tag{same exchangeability argument as \Cref{lemma: generalized sparse selective policies}}
\end{align*}
Since this holds for every $p \in \Delta(\Pi)$, we conclude $\min_{q}\max_{p} U(p, q) = \max_{p}\min_{q} U(p, q) \leq \frac{1}{K+1}$. Hence, there exists a minimizer strategy $q^{\star}\in \Delta(\Pi^K)$ such that $\max_{p}\, U(p, q^\star) \leq \frac{1}{K+1}$. Because this bound holds against any $p$, it holds in particular for the pure strategies $\delta_\pi$ for any $\pi \in \Pi$.

To prove \eqref{eq: misspecified completeness}, take $p = \delta_{\pi_0}$ to be the pure strategy playing the empirical minimizer. Since $\pi_0$ never deviates from itself, the late-stop risk against any committee is exactly zero ($\widehat{r}_N = 0$). Thus,\begin{align*}
    &U(\delta_{\pi_0}, q^\star) = 0 - \Lambda \cdot \widehat{d}_M(\pi_0) + \frac{\Lambda}{K} \mathbb{E}_{\Phi \sim q^\star}\left[ \sum_{k=1}^K \widehat{d}_M(\phi_k) \right] \leq \frac{1}{K+1}\\
    &\implies \mathbb{E}_{\Phi \sim q^\star}\left[ \sum_{k=1}^K \widehat{d}_M(\phi_k) \right] \leq K \cdot \widehat{d}_M(\pi_0) + \frac{1}{\Lambda}.
\end{align*}
To prove \eqref{eq: misspecified soundness}, take $p = \delta_\pi$ for an arbitrary target policy $\pi \in \Pi$. We have\begin{align*}
    &U(\delta_\pi, q^\star) = \mathbb{E}_{\Phi \sim q^\star}\big[\widehat{r}_N(\pi_0, \pi, \Phi)\big] - \Lambda \cdot \widehat{d}_M(\pi) + \frac{\Lambda}{K} \mathbb{E}_{\Phi \sim q^\star}\left[ \sum_{k=1}^K \widehat{d}_M(\phi_k) \right] \leq \frac{1}{K+1}\\
    &\implies \mathbb{E}_{\Phi \sim q^\star}\big[\widehat{r}_N(\pi_0, \pi, \Phi)\big] \leq \Lambda \left(\widehat{d}_M(\pi) - \frac{1}{K} \mathbb{E}_{\Phi \sim q^\star}\left[ \sum_{k=1}^K \widehat{d}_M(\phi_k) \right]\right) + \frac{1}{K+1}\\
    &\implies \mathbb{E}_{\Phi \sim q^\star}\big[\widehat{r}_N(\pi_0, \pi, \Phi)\big] \leq \Lambda \left(\widehat{d}_M(\pi) - \widehat{d}_M(\pi_0)\right) + \frac{1}{K}. \tag{$\pi_0$ minimizes $\widehat d_M$}
\end{align*}

\subsection{Proof of \Cref{thm:misspecified-deterministic-main}}

The proof closely follows the proof of \Cref{theorem: deterministic rejection sample complexity} (\Cref{subsec: proof of deterministic rejection sample complexity}); we highlight the key differences and omit steps that are identical.

\paragraph{Existence of a sparse regularized validator distribution.} By \Cref{lem:symmetrically-regularized-equilibrium}, for any $\Lambda > 0$ and integer $K \geq 1$, there exists a distribution $q^{\star}$ satisfying \eqref{eq: misspecified completeness} and \eqref{eq: misspecified soundness}. Throughout this proof, we let $\Lambda, K$ be free and optimize them at the end.

\paragraph{Bounding the empirical risk.} Unlike the deterministic non-misspecified setting—which draws $k$ independent committees to suppress test risk and relies on the version space to prevent source abstention—here we bound the expected empirical metrics over a single regularized sequence $\Phi \sim q^\star$ (recall that this is done so as to not amplify the irreducible error).

For the source data, the policy abstains if any validator in $\Phi$ disagrees with $\pi_0$. Combining the union bound with the completeness property \eqref{eq: misspecified completeness} bounds the expected empirical abstention:
\begin{align}
    \mathbb{E}_{\Phi \sim q^\star}\big[\widehat{\alpha}_M(\pi_0,\tau_{\pi_0,\Phi})\big] 
    \leq \mathbb{E}_{\Phi \sim q^\star}\left[ \sum_{k=1}^K \widehat{d}_M(\phi_k) \right] 
    \leq K \cdot \widehat{d}_M(\pi_0) + \frac{1}{\Lambda}.\label{eq: misspecified empirical abstention}
\end{align}
For the test data, evaluating against any comparator $\pi \in \Pi$, the soundness property \eqref{eq: misspecified soundness} directly bounds the expected empirical late-stop risk by the empirical suboptimality of $\pi$. In particular, taking our comparator to be $\tilde \pi = \arg\min_{\pi\in \Pi} \Delta_\pi$,
\begin{align}
    \mathbb{E}_{\Phi \sim q^\star}\big[ \widehat{r}_N(\pi_0, \tilde \pi, \Phi) \big] 
    \leq \Lambda \Big( \widehat{d}_M(\tilde \pi) - \widehat{d}_M(\pi_0) \Big) + \frac{1}{K}.\label{eq: misspecified empirical risk}
\end{align}

\paragraph{Bounding the generalization error on the expert's distribution.} We establish uniform convergence over all $|\Pi|^{K+1}$ possible combinations of base policy $\pi_0$ and validator set $\Phi$, exactly as in the non-misspecified proof. Unlike that proof, the empirical source disagreements and target late-stop risks are no longer small due to misspecification, so we use standard Hoeffding bounds rather than fast-rate bounds. Let $Z \coloneqq (K+1)\log|\Pi| + \log(4/\delta)$. By Hoeffding's inequality and a union bound, with probability at least $1-\delta$, simultaneously for all $(\pi_0, \Phi)$:
\begin{align}
    \big| d_M(\pi_0) - \widehat{d}_M(\pi_0) \big| &\le \sqrt{\frac{Z}{2m}}, \qquad \big| r_N(\pi_0, \tilde \pi, \Phi) - \widehat{r}_N(\pi_0, \tilde \pi, \Phi) \big| \le \sqrt{\frac{Z}{2n}}. \label{eq: misspecified uniform convergence}
\end{align}
For the source data, the true stopping rate satisfies $\alpha_M(\pi_0,\tau_{\pi_0,\Phi}) \le \sum_{k=1}^K d_M(\phi_k)$ by a union bound. Taking the expectation over $\Phi \sim q^\star$ and applying \eqref{eq: misspecified uniform convergence} for each $\phi_k \in \Phi$,
\begin{align*}
    \mathbb{E}_{\Phi \sim q^\star}\big[\alpha_M(\pi_0,\tau_{\pi_0,\Phi})\big] 
    &\le \mathbb{E}_{\Phi \sim q^\star}\left[ \sum_{k=1}^K \widehat{d}_M(\phi_k) \right] + K \sqrt{\frac{Z}{2m}} \tag{by \eqref{eq: misspecified uniform convergence}}\\
    &\le K \cdot \widehat{d}_M(\pi_0) + \frac{1}{\Lambda} + K \sqrt{\frac{Z}{2m}} \tag{by \eqref{eq: misspecified empirical abstention}} \\
    &\leq K \cdot \widehat{d}_M(\tilde \pi) + \frac{1}{\Lambda} + K \sqrt{\frac{Z}{2m}} \tag{$\pi_0$ minimizes $\widehat d_M$} \\
    &\le K \Delta + \frac{1}{\Lambda} + 2K \sqrt{\frac{Z}{2m}}.\tag{by \eqref{eq: misspecified uniform convergence} and definition of $\Delta$}
\end{align*}
For the test data,
\begin{align*}
    \mathbb{E}_{\Phi \sim q^\star}\big[ r_N(\pi_0, \tilde \pi, \Phi) \big] 
    &\le \mathbb{E}_{\Phi \sim q^\star}\big[ \widehat{r}_N(\pi_0, \tilde \pi, \Phi) \big] + \sqrt{\frac{Z}{2n}} \tag{by \eqref{eq: misspecified uniform convergence}}\\
    &\le \Lambda \Big( \widehat{d}_M(\tilde \pi) - \widehat{d}_M(\pi_0) \Big) + \frac{1}{K} + \sqrt{\frac{Z}{2n}}\tag{by \eqref{eq: misspecified empirical risk}}\\
    &\le \Lambda \Big( d_M(\tilde \pi) - d_M(\pi_0) \Big) + \frac{1}{K} + 2\Lambda \sqrt{\frac{Z}{2m}} + \sqrt{\frac{Z}{2n}} \tag{by \eqref{eq: misspecified uniform convergence}}\\
    &\le \Lambda \Delta + \frac{1}{K} + 2\Lambda \sqrt{\frac{Z}{2m}} + \sqrt{\frac{Z}{2n}}, \tag{definition of $\Delta$, and $-d_M(\pi_0) \leq 0$}
\end{align*}
where the last step uses $d_M(\tilde\pi) - d_M(\pi_0) \leq \Delta_{\tilde\pi} \leq \Delta$.

\paragraph{Bounding the stopping rate and regret on the learned policy's distribution.}
For the source data, we apply the same union bound and prefix coupling argument as in the proof of \Cref{theorem: deterministic rejection sample complexity}. The one difference is that $\pi_0$ may now deviate from $\pi^\star$, contributing an additional $d_M(\pi_0)$ term. Since $d_M(\pi_0) \leq d_M(\tilde\pi) + 2\sqrt{Z/2m} \leq \Delta + 2\sqrt{Z/2m}$ by \eqref{eq: misspecified uniform convergence} and the empirical minimality of $\pi_0$, taking the expectation over $q^\star$ yields:
\begin{align*}
    \mathbb{E}_{\Phi \sim q^\star}\big[\alpha_M(\pi_0, \tau_{\pi_0,\Phi})\big] 
    &\leq \mathbb{E}_{\Phi \sim q^\star}\big[ \alpha_M(\pi^\star, \tau_{\pi_0,\Phi}) \big] + d_M(\pi_0) \\
    &\leq \left( K \Delta + \frac{1}{\Lambda} + 2K \sqrt{\frac{Z}{2m}} \right) + \Delta + 2\sqrt{\frac{Z}{2m}} \\
    &= (K+1)\Delta + \frac{1}{\Lambda} + 2(K+1)\sqrt{\frac{Z}{2m}}.
\end{align*}

For the target data, let $B \coloneqq \{ \tau_{\pi_0,\{\pi^\star\}} < \tau_{\pi_0,\Phi} \}$ be the event that the deployed policy $\pi_0$ deviates from the true expert $\pi^\star$ before the validator set stops execution. As in the realizable case, the event $B$ is $\mathcal{G}_{h^{\star}}$-measurable with $h^{\star} = \min(\tau_{\pi_0,\Phi}, \tau_{\pi_0,\{\pi^\star\}})$. \Cref{lemma:prefix-coupling} therefore gives $\Pr_{N,\pi_0}[B] = \Pr_{N,\pi^\star}[B]$. 

To bound this probability on the expert's distribution, we compare $\pi^\star$ to the best-in-class policy $\tilde \pi$. On any expert trajectory, if $\pi_0$ deviates from $\pi^\star$ before $\Phi$ triggers (event $B$), then either (1) $\tilde \pi$ deviates from $\pi^\star$ at some point during the trajectory, or (2) $\tilde \pi$ perfectly matches $\pi^\star$ everywhere, meaning the deviation between $\pi_0$ and $\pi^\star$ is simultaneously a deviation between $\pi_0$ and $\tilde \pi$, and this deviation occurs before $\Phi$ triggers. This implies the set containment $B \subseteq \{ \tau_{\tilde \pi, \{\pi^\star\}} \le H \} \cup \{ \tau_{\pi_0,\Phi} > \tau_{\pi_0,\{\tilde \pi\}} \}$ on expert trajectories. Applying a union bound, we have $\Pr_{N,\pi^\star}[B] \le d_N(\tilde \pi) + r_N(\pi_0, \tilde \pi, \Phi)$. 

As in the proof of \Cref{theorem: deterministic rejection sample complexity}, the stopped regret is bounded by $\costmax$ times the probability of an unrejected deviation from the expert, so taking the expectation over $q^\star$ yields:
\begin{align*}
    \mathbb{E}_{\Phi \sim q^\star} \big[ \Regret_N(\pi_0, \tau_{\pi_0,\Phi}; c) \big] 
    &\leq \costmax \cdot \mathbb{E}_{\Phi \sim q^\star}\big[ \Pr_{N,\pi_0}[B] \big] \\
    &= \costmax \cdot \mathbb{E}_{\Phi \sim q^\star}\big[ \Pr_{N,\pi^\star}[B] \big] \\
    &\leq \costmax \left( d_N(\tilde \pi) + \mathbb{E}_{\Phi \sim q^\star}\big[ r_N(\pi_0, \tilde \pi, \Phi) \big] \right) \\
    &\leq \costmax \left( \Delta + \Lambda \Delta + \frac{1}{K} + 2\Lambda \sqrt{\frac{Z}{2m}} + \sqrt{\frac{Z}{2n}} \right).
\end{align*}
Setting $\Lambda = K = \Theta(\Delta^{-1/2})$ balances the $\Delta$-dependent and $\Delta$-independent terms in both bounds, yielding the rates in \Cref{thm:misspecified-deterministic-main}.

\subsection{Proof of Corollary~\ref{cor: misspecified optimal rates}}

Let $c_0 \coloneqq \log|\Pi| + \log(4/\delta)$. We define the $\epsilon^\star \coloneqq \max\left\{ (c_0/m)^{1/5}, (c_0/n)^{1/3} \right\}$ and set $K = \Lambda = \lceil (\Delta^{1/2} + \epsilon^\star)^{-1} \rceil$. Substituting this into \Cref{thm:misspecified-deterministic-main},
\[
    K\Delta + \frac{1}{K} \le \frac{\Delta}{\Delta^{1/2}} + (\Delta^{1/2} + \epsilon^\star) = 2\Delta^{1/2} + \epsilon^\star.
\]
For the sample estimation penalties, noting $C \le (K+1)c_0$ and $K \le (\epsilon^\star)^{-1}$, the labeled penalty is bounded by $(\epsilon^\star)^{-3/2}\sqrt{c_0/m}$. By definition, $\epsilon^\star \ge (c_0/m)^{1/5}$, which implies $\sqrt{c_0/m} \le (\epsilon^\star)^{5/2}$. Thus, the labeled penalty is at most $(\epsilon^\star)^{-3/2}(\epsilon^\star)^{5/2} = \epsilon^\star$. An identical balancing argument bounds the unlabeled penalty: $K^{1/2}\sqrt{c_0/n} \le (\epsilon^\star)^{-1/2}(\epsilon^\star)^{3/2} = \epsilon^\star$. 

Because all terms in the abstention and regret bounds of \Cref{thm:misspecified-deterministic-main} are upper-bounded by $O(\Delta^{1/2} + \epsilon^\star)$, expanding $\epsilon^\star$ yields the claimed result.

\subsection{Proof of Proposition~\ref{cor:off-policy-misspecified-rates}}

The proof is a direct adaptation of the proof of \Cref{thm:misspecified-deterministic-main}. We first redefine the suboptimality gaps and their empirical estimators with respect to the off-policy behavior policies $\pi^{\mathsf{tr}}$ and $\pi^{\mathsf{te}}$ instead of the expert $\pi^\star$:
$$d_M^{\mathsf{tr}}(\bar\pi) \coloneqq \Pr_{M, \pi^{\mathsf{tr}}}\big[\tau_{\bar\pi, \{\pi^{\mathsf{tr}}\}} \le H\big], \qquad d_N^{\mathsf{te}}(\bar\pi) \coloneqq \Pr_{N, \pi^{\mathsf{te}}}\big[\tau_{\bar\pi, \{\pi^{\mathsf{te}}\}} \le H\big]$$
$$\widehat d_M^{\mathsf{tr}}(\pi) \coloneqq \frac{1}{m}\sum_{i=1}^m \mathbf 1\big[\tau_{\pi, \{\pi^{\mathsf{tr}}\}}(S_i)\le H\big], \qquad \widehat r_N^{\mathsf{te}}(\pi_0,\pi,\Phi) \coloneqq \frac{1}{n}\sum_{j=1}^n \mathbf 1\big[\tau_{\pi_0,\Phi}(T_j) > \tau_{\pi_0,\{\pi\}}(T_j)\big].$$
Let $\pi_0 \in \arg\min_{\pi \in \Pi} \widehat d_M^{\mathsf{tr}}(\pi)$. Applying the symmetrically regularized equilibrium (\Cref{lem:symmetrically-regularized-equilibrium}) to these modified losses guarantees the existence of a distribution $q^\star$ such that:\begin{align*}
    &\mathbb{E}_{\Phi \sim q^\star}\left[\sum_{k=1}^K \widehat d_M^{\mathsf{tr}}(\phi_k)\right] \le K\widehat d_M^{\mathsf{tr}}(\pi_0) + \frac{1}{\Lambda},\\
    &\mathbb{E}_{\Phi \sim q^\star}\big[\widehat r_N^{\mathsf{te}}(\pi_0,\pi,\Phi)\big] \le \Lambda \bigl(\widehat d_M^{\mathsf{tr}}(\pi) - \widehat d_M^{\mathsf{tr}}(\pi_0)\bigr) + \frac{1}{K} \qquad \text{for every } \pi \in \Pi.
\end{align*}
From here, the mechanics follow exactly as before. Applying Hoeffding bounds introduces the $O(\sqrt{C/m})$ and $O(\sqrt{C/n})$ concentration penalties to these empirical terms. The prefix-coupling argument bounds the true completeness $\alpha_M$ by the sum of training disagreements, and the test regret by $\costmax$ times the late-stop risk against $\pi^{\mathsf{te}}$.

Finally, to bound the regret against the training demonstrator $\pi^{\mathsf{tr}}$ in the test environment, coupling $\pi^{\mathsf{tr}}$ to $\bar\pi$ and then $\bar\pi$ to $\pi^{\mathsf{te}}$ naturally introduces the discrepancy terms $d_N^{\mathsf{tr}}(\bar\pi)$ and $d_N^{\mathsf{te}}(\bar\pi)$ via the union bound, yielding the secondary result.

Finally, to obtain the optimal rates claimed in the proposition, we evaluate these bounds at the comparator $\bar\pi_{\mathrm{off}}\in \arg\min_{\bar\pi\in\Pi}\max\{d_M^{\mathsf{tr}}(\bar\pi),d_N^{\mathsf{te}}(\bar\pi)\}$. Setting $K=\Lambda=\left\lceil(\Delta_{\mathrm{off}}^{1/2}+\epsilon^\star)^{-1}\right\rceil$ with $\epsilon^\star \coloneqq \max\{(\log|\Pi|/m)^{1/5},(\log|\Pi|/n)^{1/3}\}$ and applying the identical balancing argument from the proof of Corollary~\ref{cor: misspecified optimal rates} yields the final $\widetilde{O}$ rate.

\end{document}